\documentclass{article}

\PassOptionsToPackage{numbers, compress}{natbib}
     \usepackage[preprint]{neurips_2021}

\usepackage{amsmath}
\usepackage{amsthm}
\usepackage{amssymb}
\usepackage{nccmath}
\usepackage{empheq}
\usepackage{xcolor, color, colortbl}
\usepackage{mdframed}
\usepackage{pifont}
\usepackage{enumitem}
\usepackage{bm}
\usepackage{graphicx} %
\usepackage{thmtools}
\usepackage{thm-restate}

\usepackage{apxproof}

\usepackage{nicefrac}
\usepackage{booktabs}
\usepackage{multirow}
\usepackage{rotating}

\definecolor{mydarkblue}{rgb}{0,0.08,0.45}
\usepackage[colorlinks=true,
    linkcolor=mydarkblue,
    citecolor=mydarkblue,
    filecolor=mydarkblue,
    urlcolor=mydarkblue,
    pdfview=FitH]{hyperref}
\usepackage[capitalise,nameinlink]{cleveref}
\crefname{prob}{Problem}{Problems}
\creflabelformat{prob}{#2{\upshape(#1)}#3}

\def\xx{{\boldsymbol x}}

\def\ddelta{{\boldsymbol \delta}}

\def\vv{{\boldsymbol v}}
\def\uu{{\boldsymbol u}}
\def\ww{{\boldsymbol w}}

\def\AA{{\boldsymbol A}}

\def\bF{{\boldsymbol F}}

\def\EE{{\mathbb E}\,}

\renewcommand{\gg}{\boldsymbol{g}}

\DeclareMathOperator{\sign}{sign}

\DeclareMathOperator*{\argmin}{{arg\,min}}
\DeclareMathOperator*{\argmax}{{arg\,max}}

\definecolor{myblue}{HTML}{D2E4FC}
\definecolor{Gray}{gray}{0.92}

\theoremstyle{plain}%
\newtheorem{thm}{Theorem}[section]
\newtheorem{lem}[thm]{Lemma}

\newtheorem{cor}{Corollary}

\theoremstyle{definition}

\newtheorem{defn}{Definition}[section]

\theoremstyle{remark}

\newcommand{\E}{{\mathbb{E}}}

\newcommand{\C}{{\mathbb{C}}}
\newcommand{\R}{{\mathbb{R}}}
\newcommand{\N}{{\mathcal{N}}}

\newcommand{\I}{{\mathbb{I}}}

\newcommand{\dxy}[2]{{\frac{\partial {#1}}{\partial #2}}}

\newcommand{\twocolfigwidth}{0.47\linewidth}

\graphicspath{{./figures/}}

\usepackage[utf8]{inputenc} %
\usepackage[T1]{fontenc}    %
\usepackage{hyperref}       %
\usepackage{url}            %
\usepackage{booktabs}       %
\usepackage{amsfonts}       %
\usepackage{nicefrac}       %
\usepackage{microtype}      %
\usepackage{xcolor}         %
\usepackage{subcaption}
\usepackage{wrapfig}
\usepackage{algorithm}
\usepackage{algorithmic}

\title{Bridging the Gap Between Adversarial Robustness and Optimization Bias}

\author{%
Fartash Faghri$^{2, 3}$\thanks{Work done during an internship at Google Research. Code available at: \url{https://github.com/fartashf/robust_bias}.
Correspondence to: Fartash Faghri \textless{}faghri@cs.toronto.edu\textgreater{}.}
\And
Sven Gowal$^4$
\And
Cristina Vasconcelos$^1$
\vspace*{-.8cm}
\AND
David J. Fleet$^{1,2,3}$
\And
Fabian Pedregosa$^1$
\And
Nicolas Le Roux$^{1,5}$
\AND
\vspace{-.7cm}
\\
$^1$Google Research\quad
$^2$University of Toronto\quad
$^3$Vector Institute\quad
$^4$DeepMind\quad
$^5$Mila%
}

\begin{document}

\maketitle

\begin{abstract}
\vspace*{-0.25cm}

We demonstrate that the choice of optimizer,
neural network architecture,
and  regularizer significantly affect the adversarial 
robustness of linear neural networks, providing guarantees without
the need for adversarial training.
To this end, we revisit a known result linking maximally robust classifiers and minimum norm solutions, and combine it with recent results on the implicit bias of optimizers.
First, we show that, under certain conditions, it is possible to achieve both perfect standard accuracy and a certain degree of robustness,
simply by training an overparametrized model using the implicit bias of the optimization. In that regime, there is a direct relationship between the type of the optimizer and the attack to which the model is robust.
To the best of our knowledge, this work is the first to
study the impact of optimization methods such as sign gradient descent and proximal methods on adversarial robustness.
Second, we characterize the robustness of
linear convolutional models, showing that they resist attacks subject to a constraint on the Fourier-$\ell_\infty$ norm.
To illustrate these findings we design a novel
Fourier-$\ell_\infty$ attack that finds
adversarial examples with controllable frequencies.
We evaluate Fourier-$\ell_\infty$ robustness of adversarially-trained
deep CIFAR-10 models from the standard RobustBench
benchmark and visualize adversarial perturbations.

\end{abstract}

\section{Introduction}
\label{sec:intro}

Deep neural networks achieve high accuracy on standard test
sets, yet \citet{szegedy2013intriguing} showed that
any natural input correctly classified by a neural network
can be modified with adversarial perturbations
that fool the network into misclassification,
even when such perturbations are constrained
to be small enough that do not
significantly affect human perception.
Adversarial training improves model robustness through training
on adversarial samples~\citep{goodfellow2014explaining}
and can be interpreted as approximately
solving a saddle-point problem~\citep{madry2017towards}.
Adversarial training is the state-of-the-art approach to adversarial
robustness~\citep{gowal2020uncovering, croce2020robustbench} and
alternative approaches are more likely to exhibit
spurious robustness~\citep{tramer2020adaptive}.
Nevertheless, adversarial training is computationally expensive 
compared to standard training, as it involves an alternated optimization.
Adversarial training also exhibits a trade-off between standard generalization 
and adversarial robustness. 
That is, it achieves improved {\it robust accuracy}, on
adversarially perturbed data, at the expense of {\it standard accuracy}, the
probability of correct predictions on natural data \citep{tsipras2018robustness}.
This adversarial robustness trade-off has been shown to be intrinsic in a
number of toy examples~\citep{fawzi2018analysis}, independent of the
learning algorithm in some cases~\citep{schmidt2018adversarially}.
Alternatives to adversarial training have been proposed to
reduce this trade-off, but a gap 
remains in practice~\citep{zhang2019theoretically}.

Here we consider connections between the adversarial robustness trade-off 
and optimization biases in training overparametrized models.
Deep learning models can often achieve interpolation, i.e., they have 
the capacity to exactly fit the training data~\citep{zhang2016understanding}. 
Their ability to generalize well in such cases has been attributed 
to an implicit bias toward simple
solutions~\citep{gunasekar2018characterizing,hastie2019surprises}.

Our main contribution is 
to
connect two large bodies of work on adversarial robustness
and optimization bias.
Focusing on models that achieve interpolation,
we use the formulation of a {\it Maximally Robust Classifier}
from robust optimization~\citep{ben2009robust}.
We theoretically demonstrate that the choice of optimizer
(\cref{cor:implicit_robust}), neural network architecture
(\cref{cor:max_robust_to_min_norm_conv}), and regularizer
(\cref{cor:reg_is_max_margin}), significantly affect
the adversarial robustness of linear neural networks.
Even for linear models, the impact of these choices
had not been characterized precisely prior to our work.
We observe that, in contrast to adversarial training,
under certain
conditions we can find maximally robust classifiers at no additional
computational cost.

Based on our theoretical results on the robustness of linear convolutional
models to Fourier attacks,
we introduce a new class of attacks in the Fourier
domain. In particular, we design the Fourier-$\ell_\infty$
attack and illustrate our theoretical results. Extending to
non-linear models, we attack adversarially-trained deep models on
CIFAR-10 from RobustBench benchmark~\citep{croce2020robustbench}
and find low and high frequency adversarial perturbations
by directly controlling
spectral properties through Fourier constraints.
This example demonstrates how
understanding maximal robustness of linear models is a
stepping stone to understanding and guaranteeing robustness
of non-linear models.

\section{No Trade-offs with Maximally Robust Classifiers}
\vspace*{-5pt}
\label{sec:def_robust}

We start by defining adversarial robustness and
the robustness trade-off in adversarial training.
Then, \cref{sec:max_robust} provides an alternative formulation to
adversarial robustness that avoids the robustness trade-off.
Let
$\mathcal{D}=\{(\xx_i, y_i)\}_{i=1}^n$
denote a training set sampled I.I.D.\ from a distribution,
where
${\xx_i\in \R^d}$ are features and ${y_i\in \{-1,+1\}}$
are binary labels.
A binary classifier is a function
 ${\varphi : \R^d \rightarrow \R}$, and its
prediction on an input $\xx$ is given by
$\sign(\varphi(\xx))\in\{-1,+1\}$.
The aim in supervised
learning is to find a classifier that accurately classifies
the training data and generalizes to unseen test data.
One standard framework for training a classifier is
Empirical Risk Minimization (ERM),
$\argmin_{\varphi\in\Phi}
\mathcal{L}(\varphi)$,
where
$\mathcal{L}(\varphi)
\coloneqq \E_{(\xx,y) \sim \mathcal{D}} \,
\zeta(y\varphi(\xx))$,
$\Phi$ is a family of classifiers,
and $\zeta:\R\rightarrow\R^+$
is a loss function that we assume to be
strictly monotonically decreasing to $0$,
i.e.,\ $\zeta' < 0$.
Examples are the exponential loss, $\exp{(-\hat{y}y)}$,
and the logistic loss, $\log{(1+\exp{(-\hat{y}y)})}$.

Given a classifier, an adversarial perturbation
$\ddelta \in \R^d$ is any small perturbation that
changes the model prediction, i.e.,
$\sign(\varphi(\xx+\ddelta))\neq
\sign(\varphi(\xx)),\,
\|\ddelta\|\leq\varepsilon$,
where $\|\cdot\|$ is a norm on $\R^d$,
and $\varepsilon$ is an arbitrarily chosen constant.
It is common to use norm-ball constraints to
ensure perturbations are small (e.g., imperceptible in images)
but other constraints exist~\citep{brown2017adversarial}.
Commonly used are the $\ell_p$ norms, 
$\|\vv\|_p=\left(\sum_{i=0}^{d-1} [\vv]_i^p\right)^{1/p}$,
where $[\vv]_i$ denotes
the $i$-th element of a vector $\vv$,
for $i=0,\ldots,d-1$.
In practice, an adversarial perturbation, $\ddelta$,
is found as an approximate solution
to the following optimization problem,
\vspace*{-4pt}
\begin{equation}
\label{eq:find_adv}
\max_{\ddelta : \|\ddelta\|\leq \varepsilon}
\zeta(y\varphi(\xx+\ddelta))\,.
\end{equation}
Under certain conditions, closed form solutions exist
to the optimization problem in \eqref{eq:find_adv}.
For example, \citet{goodfellow2014explaining} observed that
the maximal $\ell_\infty$-bounded adversarial perturbation
against a linear model
(i.e.\ one causing the maximum change in the output)
is the sign gradient direction scaled by $\varepsilon$.

\citet{madry2017towards} defined an adversarially robust classifier
as the solution to the saddle-point optimization problem,
\vspace*{-4pt}
\begin{equation}
\label{eq:saddle_point}
\argmin_{\varphi\in\Phi}\, \E_{(\xx,y) \sim \mathcal{D}}
\max_{\ddelta : \|\ddelta\|\leq \varepsilon}
\zeta(y\varphi(\xx+\ddelta))\,.
\end{equation}
The saddle-point adversarial robustness problem
is the robust counter-part to empirical risk minimization where the expected loss is minimized on
worst-case adversarial samples defined as solutions
to \eqref{eq:find_adv}.
Adversarial Training~\citep{goodfellow2014explaining}
refers to solving \eqref{eq:saddle_point}
using an alternated optimization. It is
computationally expensive because
it often requires solving
\eqref{eq:find_adv} many times.

The main drawback of defining the adversarially robust classifier 
using \eqref{eq:saddle_point}, and a drawback of
adversarial training,
is that the parameter $\varepsilon$
needs to be known or tuned.
The choice of $\varepsilon$ controls a trade-off
between standard accuracy on samples
of the dataset $\mathcal{D}$
versus the robust accuracy, i.e., the accuracy on adversarial samples.
At one extreme $\varepsilon=0$, where
\eqref{eq:saddle_point} reduces to ERM.
At the other, 
as $\varepsilon\rightarrow\infty$,
all inputs in $\R^d$ are within the $\varepsilon$-ball
of every training point and can be an adversarial input.
The value of the inner max in \eqref{eq:saddle_point}
for a training point $\xx,y$ is the loss of the most confident prediction over $\R^d$ that is predicted as
$-y$. For large enough $\varepsilon$, the solution to
\eqref{eq:saddle_point} is a classifier predicting
the most frequent label, i.e.,\ 
$\varphi(\cdot)=p^\ast$, where $p^\ast$ is the solution to
$\argmin_p n_{-1}\zeta(-p)+n_{+1}\zeta(p)$, and
$n_{-1}, n_{+1}$ are the number of negative and positive
training labels.

Robust accuracy is often a complementary generalization
metric to standard test accuracy.
In practice, we prefer a classifier that is accurate on the test set,
and that additionally, achieves maximal robustness.
The saddle-point formulation makes this challenging
without the knowledge of the maximal $\varepsilon$.
This trade-off has been studied in various works~\citep{
tsipras2018robustness, zhang2019theoretically, fawzi2018adversarial,
fawzi2018analysis, schmidt2018adversarially}.
Regardless of the trade-off imposed by $\varepsilon$,
adversarial training
is considered to be the state-of-the-art for adversarial
robustness. The evaluation is
based on the robust accuracy achieved at
fixed $\varepsilon$'s
even though the standard accuracy is usually lower
than a comparable non-robust model~\citep{croce2020robustbench,
gowal2020uncovering}.

\vspace*{-5pt}
\subsection{Maximally Robust Classifier}
\label{sec:max_robust}
\vspace*{-5pt}
In order to avoid the trade-off imposed by $\varepsilon$
in adversarial robustness, we revisit a definition from
robust optimization.
\begin{defn}%
\label{def:max_robust}
A \textbf{Maximally Robust Classifier} (\citet{ben2009robust}) is a solution to
\vspace*{-1pt}
\begin{align}
\label{eq:max_robust0}
    \argmax_{\varphi\in\Phi} \{\varepsilon\,|\,y_i \varphi(\xx_i + \ddelta) > 0 
    ,\,~\forall i,\|\ddelta\| \leq \varepsilon\}\,.
\end{align}
\end{defn}
\vspace*{-10pt}
Compared with the saddle-point formulation~\eqref{eq:saddle_point},
$\varepsilon$ in \eqref{eq:max_robust0}
is not an arbitrary constant. Rather, it is
maximized as part of the optimization problem.
Moreover, the maximal $\varepsilon$ in this definition does
not depend on a particular loss function.
Note, a maximally robust classifier
is not necessarily unique.

The downside of \eqref{eq:max_robust0}
is that the formulation requires
the training data to be separable
so that \eqref{eq:max_robust0} is non-empty,
i.e.\ there exists $\varphi\in\Phi$
such that $\forall i, y_i\varphi(\xx_i)>0$.
In most deep learning settings,
this is not a concern as models are large
enough that they
can interpolate the training data, i.e.
for any dataset there exists $\varphi$ such that $\varphi(\xx_i)=y_i$.
An alternative formulation is to modify
the saddle-point problem and
include an outer maximization on $\varepsilon$ by allowing
a non-zero slack loss. However, the new slack loss
reimposes a trade-off between standard and robust accuracy
(See \cref{sec:max_robust_gen}).

One can also show that adversarial training, i.e.,
solving the saddle-point problem \eqref{eq:saddle_point}, 
does not necessarily find a maximally robust classifier.
To see this, suppose we are given the maximal
$\varepsilon$ in \eqref{eq:max_robust0}.
Further assume the minimum of \eqref{eq:saddle_point} is non-zero.
Then the cost in the saddle-point problem does not distinguish
between the following two models: 1) a model that makes no
misclassification errors but has low confidence,
i.e.
$\forall i,\, 0<\max_{\ddelta} y_i \varphi(\xx_i+\ddelta)\leq c_1$
for some small $c_1$
2) a model that classifies a training point, $\xx_j$, incorrectly
but is highly confident on all other training data and adversarially
perturbed ones, i.e.
$\forall i\neq j,\,
0<c_2 < \max_{\ddelta} y_i\varphi(\xx_i+\ddelta)$.
The second model can incur a loss $n\zeta(c_1)-(n-1)\zeta(c_2)$
on $\xx_j$ while being no worse than the first model according
to the cost of the saddle-point problem.
The reason is another trade-off between
standard and robust accuracy caused by
taking the expectation over data points.

\subsection{Linear Models: Maximally Robust is the Minimum Norm Classifier}
\vspace*{-3pt}
Given a dataset and a norm, what is the maximally robust
linear classifier with respect to that norm? In this section, we revisit a result from~\citet{ben2009robust}
for classification.

\begin{defn}[Dual norm]
Let $\|\cdot\|$ be a norm on $\R^{n}.$ The associated \emph{dual norm}, denoted $\| \cdot \|_\ast$, is defined as
$\|\ddelta\|_\ast = \sup_{\xx}\{ |\langle \ddelta, \xx \rangle|\; |\; \|\xx\| \leq 1 \}\;$.
\end{defn}

\begin{defn}[Linear Separability]
    \label{def:sep}
    We say a dataset is linearly separable if there exists $\ww,b$ such that
    ${y_i (\ww^\top \xx_i +b)> 0}$ for all $i$.
\end{defn}

\begin{restatable}[Maximally Robust Linear Classifier (\citet{ben2009robust}, \S12)]{lem}{lemmaxrobust}
\label{lem:max_robust_to_min_norm_linear}
For linear models and linearly separable data,
the following problems are equivalent; i.e., 
from a solution of one, a solution of the other is readily found.
\vspace*{-3pt}
\begin{align}
\label{eq:max_robust}
   \text{Maximally robust classifier:}\quad&
   \!\! \argmax_{\ww,b} \{\varepsilon\,|\,
    y_i (\ww^\top (\xx_i + \ddelta) +b) > 0 
    ,\,~\forall i,\|\ddelta\| \leq \varepsilon\}\, ,\\
\label{eq:max_margin}
   \text{Maximum margin classifier:}\quad&
    \argmax_{\ww,b : \|\ww\|_\ast\leq 1} \{\varepsilon\,|\,
    y_i (\ww^\top  \xx_i + b)
    \geq \varepsilon,\, ~\forall i\}\,,\\
\label{eq:min_norm}
   \text{Minimum norm classifier:}\quad&
    \argmin_{\ww,b} \{\|\ww\|_\ast\,|\,
    y_i (\ww^\top  \xx_i+b) \geq 1,\, ~\forall i\}\,.
\end{align}
\vspace*{-2pt}
The expression $\min_i y_i (\ww^\top \xx_i+b)/\|\ww\|$
is the margin
of a classifier $\ww$ that is the distance of the nearest training
point to the classification boundary,
i.e.\ the line ${\{\vv:\ww^\top \vv=-b\}}$.
\end{restatable}

We provide a proof for
general norms based on \citet{ben2009robust} in \cref{proof:max_robust_to_min_norm_linear}.
Each formulation in \cref{lem:max_robust_to_min_norm_linear}
is connected to a wide array of results that can be transferred
to other formulations. Maximally robust classification is
one example of a problem in robust optimization that can
be reduced and solved efficiently. Other problems
such as robust regression as well as
robustness to correlated input perturbations have been studied
prior to deep learning~\citep{ben2009robust}.

On the other hand, maximum margin and minimum norm
classification have long been popular because of their
generalization guarantees.
Recent theories for overparametrized models
link the margin and the norm of a model
to generalization~\citep{hastie2019surprises}.
Although the tools are different, connecting the margin and
the norm of a model has also been the basis of generalization
theories for Support Vector Machines 
and AdaBoost~\citep{shawe1998structural,telgarsky2013margins}.
Maximum margin classification does not require linear separability,
because there can exist a classifier with $\varepsilon<0$ that
satisfies the margin constraints.
Minimum norm classification is the easiest formulation
to work with in practice
as it does not rely on $\varepsilon$ nor $\ddelta$ and
minimizes a function of the weights subject to a set of
constraints.

In what follows, we use
\cref{lem:max_robust_to_min_norm_linear}
to transfer recent results about minimum norm classification
to maximally robust classification. These results 
have been the basis for explaining generalization properties
of deep learning models~\citep{hastie2019surprises,nakkiran2019deep}.

\vspace*{-5pt}
\section{Implicit Robustness of Optimizers}
\vspace*{-5pt}
\label{sec:implicit_bias}

The most common approach to empirical risk minimization (ERM) is through 
gradient-based optimization. 
As we will review shortly, \citet{gunasekar2018characterizing} showed
that gradient  descent, and more generally steepest descent methods, 
have an implicit bias towards minimum norm solutions. 
From the infinitely many solutions that minimize 
the empirical risk,
we can characterize the one found by steepest descent.
Using \cref{lem:max_robust_to_min_norm_linear}, we show that
such a classifier is also maximally robust
w.r.t. a specific norm.

Recall that ERM is defined as
$\argmin_{\varphi\in\Phi}\mathcal{L}(\varphi)$, where
$\mathcal{L}(\varphi)
= \E_{(\xx,y) \sim \mathcal{D}} \zeta(y\varphi(\xx))$.
Here we assume $\mathcal{D}$ is a finite dataset of size $n$.
For the linear family of functions, we write $\mathcal{L}(\ww,b)$.
Hereafter, we rewrite the loss as
$\mathcal{L}(\ww)$ and use an augmented representation with
a constant $1$ dimension.
For linearly separable data and overparametrized models
($d>n$), there exist infinitely many linear
classifiers that minimize the empirical
risk~\citep{gunasekar2018characterizing}.
We will find it convenient to ignore the scaling
and focus on the normalized vector $\ww/\|\ww\|$, i.e.\
the direction of $\ww$.
We will say that the sequence $\ww_1, \ww_2, \ldots$ 
converges in direction to a vector $\vv$
if $\lim_{t\rightarrow\infty} \ww_t/\|\ww_t\|=\vv$.

\vspace*{-5pt}
\subsection{Steepest Descent on Fully-Connected Networks}
\vspace*{-5pt}

\begin{defn}[Steepest Descent]
Let $\|\cdot\|$ denote a norm, $f$ a function to be minimized, and
$\gamma$ a step size. The steepest descent method associated with this norm finds
\vspace*{-4pt}
\begin{align}\label{eq:steepest_descent}
\ww_{t+1} &= \ww_t + \gamma \Delta\ww_t,\nonumber\\
\text{where}~~\Delta \ww_{t} &\in \argmin_\vv \langle\nabla f(\ww_t), \vv \rangle
+ \frac{1}{2}\|\vv\|^2\,.
\end{align}
The steepest descent step, $\Delta \ww_{t}$, can be equivalently written as
${-\|\nabla f(\ww_t)\|_\ast\, g_{\text{nst}}}$,
where
$g_{\text{nst}} \in \argmin\left\{ \langle\nabla f(\ww_t), \vv\rangle\; |\; \|\vv\| = 1\right\}$.
A proof can be found in \citep[\S9.4]{boyd2004convex}.
\vspace*{-10pt}
\end{defn}

\paragraph{Remark.} For some $p$-norms,
steepest descent steps have closed form expressions.
Gradient Descent (GD) is steepest descent w.r.t.\ $\ell_2$ norm 
where ${- \nabla f(\ww_t)}$ is a steepest descent step.
Sign gradient descent is steepest descent w.r.t.\
$\ell_\infty$ norm where  
${- \|\nabla f(\ww_t)\|_1 \sign(\nabla f(\ww_t))}$ is a steepest descent step.
Coordinate Descent (CD) is steepest descent w.r.t.\ $\ell_1$ norm where
${- \nabla f(\ww_t)_i {\boldsymbol{e}}_i }$ is a steepest descent step
($i$ is the coordinate for which the gradient has the largest absolute magnitude).

\begin{thm}[Implicit Bias of Steepest Descent (\citet{gunasekar2018characterizing} (Theorem 5))]
\label{thm:gunasekar_linear}
For any separable dataset $\{\xx_i, y_i\}$
and any norm $\|\cdot\|$,
consider the steepest descent updates
from \eqref{eq:steepest_descent} for minimizing
the empirical risk $\mathcal{L}(\ww)$ (defined in \cref{sec:def_robust})
with the exponential loss, $\zeta(z)=\exp{(-z)}$. For all
initializations $\ww_0$, and
all bounded step-sizes satisfying a known upper bound,
the iterates $\ww_t$ satisfy
\vspace*{-7pt}
\begin{equation}
\label{eq:gunasekar_linear}
    \lim_{t \to \infty}
    \min_i \frac{y_i \ww_t^\top \xx_i}{\|\ww_t\|}
    = \max_{\ww : \|\ww\| \leq 1}
    \min_i y_i \ww^\top \xx_i\, .
\end{equation}
In particular, if a unique maximum margin classifier
$\ww_{\|\cdot\|}^\ast=\argmax_{\ww : \|\ww\| \leq 1} \min_i y_i \ww^\top \xx_i$
exists, the limit direction converges to it, i.e.\
$\lim_{t \to \infty}\frac{\ww_t}{\|\ww_t\|}
=\ww^\ast_{\|\cdot\|}$.
\end{thm}

In other words, the margin converges to the maximum margin
and if the maximum margin classifier is unique, the
iterates converge in direction to $\ww^\ast_{\|\cdot\|}$.
We use this result to derive our \cref{cor:implicit_robust}.

\begin{cor}[Implicit Robustness of Steepest Descent]
\label{cor:implicit_robust}
For any linearly separable dataset and any norm $\|\cdot\|$,
steepest descent iterates
minimizing the empirical risk, $\mathcal{L}(\ww)\,$,
satisfying the conditions of \cref{thm:gunasekar_linear},
converge in direction to a maximally robust classifier,
\begin{align*}
    \argmax_{\ww}
    \{\varepsilon\,|\,y_i \ww^\top (\xx_i + \ddelta) > 0 
    ,\,~\forall i,\,\|\ddelta\|_\ast \leq \varepsilon\}\,.
\end{align*}
In particular, a maximally robust classifier against
$\ell_1$, $\ell_2$, and $\ell_\infty$ is reached, respectively, by
sign gradient descent, gradient descent, and coordinate descent.
\vspace*{-10pt}
\end{cor}

\vspace*{-5pt}
\begin{proof}
By \cref{thm:gunasekar_linear},
the margin of the steepest descent iterates,
$\min_i \frac{y_i \ww_t^\top \xx_i}{\|\ww_t\|}$ ,
converges as $t \to \infty$ to the maximum margin,
$\max_{\ww : \|\ww\| \leq 1} \min_i y_i \ww^\top \xx_i$.
By  \cref{lem:max_robust_to_min_norm_linear},
any maximum margin classifier w.r.t. $\|\cdot\|$
gives a maximally robust classifier w.r.t. $\|\cdot\|_\ast$.
\vspace*{-5pt}
\end{proof}

\cref{cor:implicit_robust} implies that for
overparametrized linear models,
we obtain guaranteed robustness by an appropriate choice of
optimizer without the additional cost
and trade-off of adversarial training.
We note that \cref{thm:gunasekar_linear} and
\cref{cor:implicit_robust},
characterize linear models, but do not account for the bias $b$.
We can close the gap with an augmented input representation,
to include the bias explicitly.
Or one could preprocess the data, 
removing the mean before training.

To extend \cref{cor:implicit_robust} to deep learning models
one can use generalizations of \cref{thm:gunasekar_linear}.
For the special case of gradient descent,
\cref{thm:gunasekar_linear} has been generalized to multi-layer
fully-connected linear networks and a larger family
of  strictly monotonically decreasing loss functions
including the logistic loss~\citep[Theorem 2]{nacson2019convergence}.

\subsection{Gradient Descent on Linear Convolutional Networks}
\vspace*{-5pt}
\label{sec:implicit_conv}

In this section, we show that even for linear models,
the choice of the architecture affects implicit robustness,
which gives another alternative for achieving maximal robustness.
We use a generalization of
\cref{thm:gunasekar_linear} to linear convolutional models.

\begin{defn}[Linear convolutional network]
An $L$-layer convolutional network with
$1$-D circular convolution is parameterized using
weights of $L-1$ convolution layers,
$\ww_1,\ldots,\ww_{L-1}\in\R^d$,
and weights of a final linear layer, $\ww_L\in\R^d$,
such that the linear mapping of the network is
\vspace*{-2pt}
\begin{equation*}
\varphi_{\text{conv}}(\xx;\ww_1,\ldots,\ww_L)
\coloneqq
\ww_L^\top (\ww_{L-1} \star \cdots (\ww_1 \star \xx))\, .
\end{equation*}
Here, circular convolution is defined as
$ [\ww\star\xx]_i
\coloneqq
\frac{1}{\sqrt{d}} \sum_{k=0}^{d-1}
[\ww]_{\overline{-k}} [\xx]_{\overline{i+k}}$, where
$[\vv]_i$ denotes the $i$-th element of a vector $\vv$
for $i=0,\ldots,d-1$,
and $\overline{\phantom{i}i \phantom{i}}=i\bmod d$.~\footnote{
We use the usual definition of circular
convolution in signal processing, rather than cross-correlation,
$\ww^\downarrow \star \xx$ with $[\vv^\downarrow]_i=[\vv]_{\overline{-i}}$, which is used in deep learning literature, but not associative.}
\vspace*{-5pt}
\end{defn}

A linear convolutional network is equivalent
to a linear model with weights
$ \ww= \ww_L\star(\cdots\star (\ww_{2}\star\ww_1)) $ because
of the associative property of convolution.
In particular, for two-layer linear convolutional networks
$\ww=\ww_2\star\ww_1$.

\begin{defn}[Discrete Fourier Transform]
$\mathcal{F}(\ww)\in\C^d$ denotes the Fourier coefficients of $\ww$
where
$[\mathcal{F}(\ww)]_d = \frac{1}{\sqrt{d}}
\sum_{k=0}^{d-1} [\ww]_k \exp(-\frac{2\pi j}{d} k d)$
and $j^2=-1$.
\end{defn}

\begin{thm}[Implicit Bias towards Fourier Sparsity
(\citet{gunasekar2018implicit}, Theorem 2, 2.a)]
\label{thm:gunasekar_conv}
Consider the family of
$L$-layer linear convolutional networks and
the sequence of gradient descent iterates, $\ww_t$,
minimizing the empirical risk, $\mathcal{L}(\ww)$, with
the exponential loss, $\exp{(-z)}$.
For almost all linearly separable datasets
under known conditions on the step size and
convergence of iterates,
$\ww_t$  converges in direction
to the classifier minimizing
the norm of the Fourier coefficients given by
\vspace{-2pt}
\begin{align}
\argmin_{\ww_1,\ldots,\ww_L} \{ \|\mathcal{F}(\ww)\|_{2/L}\, |\,
y_i \langle\ww, \xx_i\rangle\geq 1,\,\forall i\}.
\end{align}
In particular, for two-layer linear convolutional networks
the implicit bias is towards the solution with minimum $\ell_1$ norm of the Fourier coefficients,
$\|\mathcal{F}(\ww)\|_1$.
For $L>2$, the convergence is to a first-order stationary
point.
\vspace*{-5pt}
\end{thm}

We use this result to derive our \cref{cor:max_robust_to_min_norm_conv}.

\begin{restatable}[Maximally Robust to Perturbations with Bounded Fourier Coefficients]{cor}{cormaxrobustconv}
\label{cor:max_robust_to_min_norm_conv}
Consider the family of two-layer linear convolutional networks
and the gradient descent iterates, $\ww_t$,
minimizing the empirical risk.
For almost all linearly separable datasets under conditions
of \cref{thm:gunasekar_conv}, $\ww_t$ converges in direction
to a maximally robust classifier,
\vspace*{-2pt}
\begin{align*}
    \argmax_{\ww_1,\ldots,\ww_L}
    \{\varepsilon\,|\,y_i
    \varphi_{\text{conv}}(\xx_i+\ddelta;\{\ww_l\}_{l=1}^L) > 0
    ,\,~\forall i,\,\|\mathcal{F}(\ddelta)\|_\infty \leq \varepsilon\}\,.
\end{align*}
\end{restatable}
\vspace*{-5pt}

Proof in \cref{proof:max_robust_to_min_norm_conv}.
\cref{cor:max_robust_to_min_norm_conv} implies
that, at no additional cost,
linear convolutional models are already maximally robust,
but w.r.t. perturbations in the Fourier domain.
We call attacks with $\ell_p$ constraints in the
Fourier domain \textit{Fourier-$\ell_p$} attacks.
\cref{sec:unit_norm_balls} depicts various norm-balls
in $3$D to illustrate the significant geometrical difference
between the Fourier-$\ell_\infty$ and other commonly used norm-balls for adversarial robustness.
One way to understand \cref{cor:max_robust_to_min_norm_conv}
is to think of perturbations that succeed in fooling a
linear convolutional network.
Any such
adversarial perturbation must have at least one frequency beyond the maximal robustness of the model.
This condition is satisfied for perturbations
with small $\ell_1$ norm in the spatial domain,
i.e.,\ only a few pixels are perturbed.

\subsection{Fourier Attacks}
\vspace*{-5pt}

\begin{wrapfigure}[10]{r}{0.5\textwidth}
\vspace*{-35pt}
\begin{minipage}{0.5\textwidth}
\begin{algorithm}[H]
   \caption{Fourier-$\ell_\infty$ Attack
   (see \cref{sec:fourier_ops})}
   \label{alg:fourier_linf_attack}
\begin{algorithmic}
   \STATE {\bfseries Input:} data $\xx$, label $y$,
   loss function $\zeta$,
   classifier $\varphi$,
   perturbation size $\varepsilon$,
   number of attack steps $m$,
   dimensions $d$,
   Fourier transform $\mathcal{F}$
   \FOR{$k=1$ {\bfseries to} $m$}
       \STATE $\hat{\gg} = \mathcal{F}(\nabla_\xx \zeta(y\varphi(\xx)))$
       \STATE $[\ddelta]_i =
       \varepsilon\frac{[\hat{\gg}]_i}{|[\hat{\gg}]_i|},\,
       \forall i\in \{0,\ldots,d-1\}$
    \STATE $\xx = \xx + \mathcal{F}^{-1}(\ddelta)$
   \ENDFOR
\end{algorithmic}
\end{algorithm}
\end{minipage}
\end{wrapfigure}

The predominant motivation for designing new attacks is to
\emph{fool} existing models. 
In contrast, our results characterize the attacks 
that existing models \emph{perform best} 
against, as measured by maximal robustness.
Based on \cref{cor:max_robust_to_min_norm_conv} we
design the Fourier-$\ell_p$ attack to verify our results.
Some adversarial attacks exist with
Fourier constraints~\citep{tsuzuku2019structural,guo2019low}.
\citet{sharma2019effectiveness} proposed a Fourier-$\ell_p$ 
attack that includes Fourier constraints
in addition to $\ell_p$ constraints in the spatial domain.
Our theoretical results suggest a more general class of
attacks with only Fourier constraints.

The maximal $\ell_p$-bounded adversarial perturbation
against a linear model in \eqref{eq:find_adv} consists of real-valued
constraints with a closed form solution.
In contrast, maximal Fourier-$\ell_p$ has complex-valued constraints.
In \cref{sec:fourier_ops} we derive the Fourier-$\ell_\infty$
attack in closed form for linear models and provide the pseudo-code
in \cref{alg:fourier_linf_attack}.
To find perturbations as close as possible to
natural corruptions such as
blur, $\varepsilon$ can be a matrix of constraints
that is multiplied elementwise by $\ddelta$.
As our visualizations in \cref{fig:image_attack} show,
adversarial perturbations under bounded
Fourier-$\ell_\infty$ 
can be controlled to be high frequency and
concentrated on subtle details of the image, or low frequency
and global.
We observe that high frequency Fourier-$\ell_\infty$ attacks
succeed more easily with smaller perturbations compared
with low frequency attacks.
The relative success of our
band-limited Fourier attacks matches
the empirical observation
that the amplitude spectra
of common $\ell_p$ attacks are largely 
band-limited as such attacks succeed more easily~\citep{yin2019fourier}.

\begin{figure*}[t]
    \centering
    \begin{subfigure}[b]{.31\textwidth}
    \begin{tabular}{c}
    $\xx$\hfill
    $\xx+\ddelta$\hfill
    \,$\ddelta$\hfill
    \\
    \includegraphics[width=.95\textwidth]{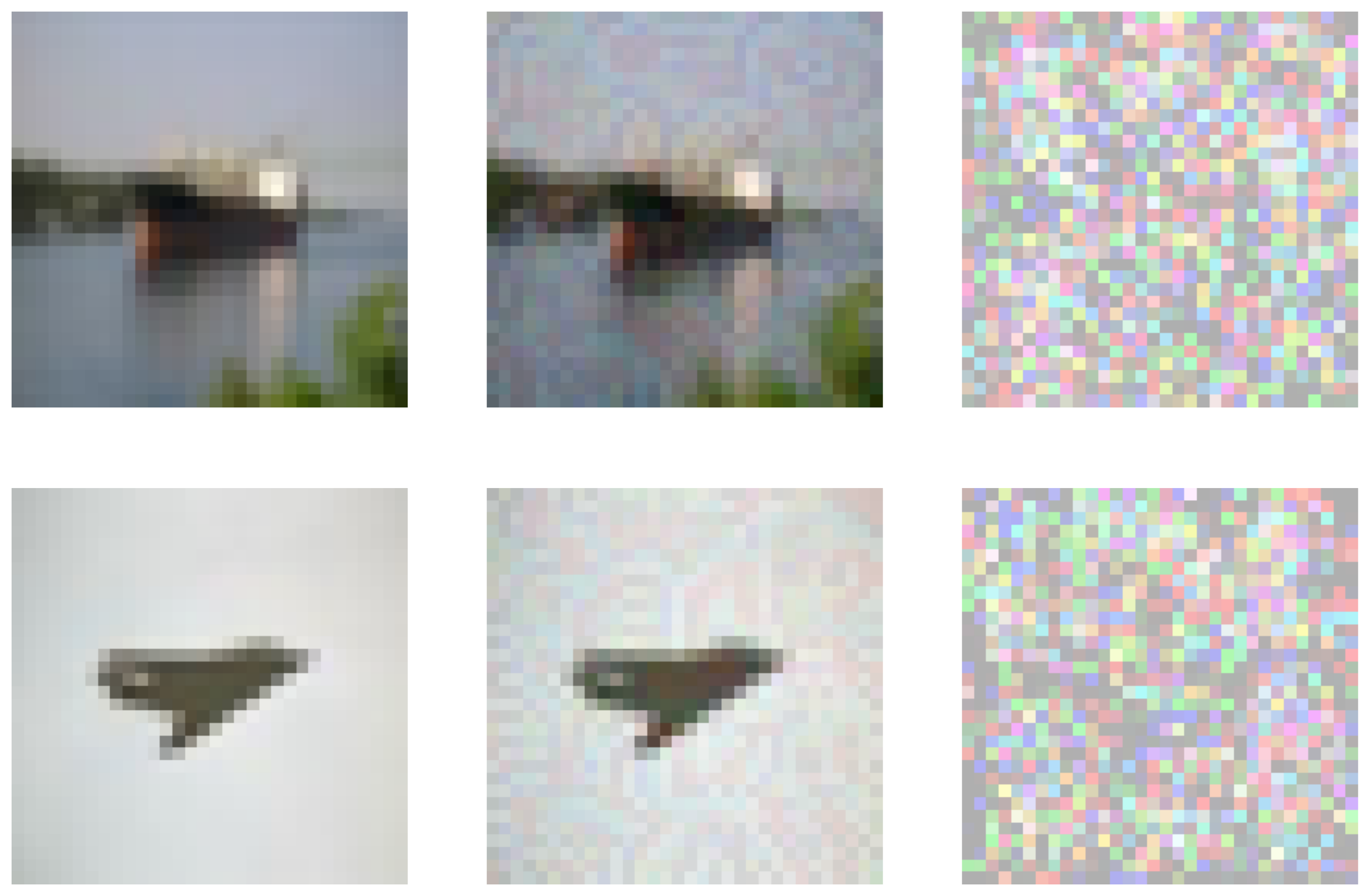}
    \end{tabular}
    \caption{$\ell_\infty$ attack}
    \label{fig:image_attack_linf}
    \end{subfigure}
    \hfill
    \begin{subfigure}[b]{.21\textwidth}
    \begin{tabular}{c}
    $\xx+\ddelta$\hfill
    \,$\ddelta$\hfill
    \\
    \includegraphics[width=.95\textwidth]{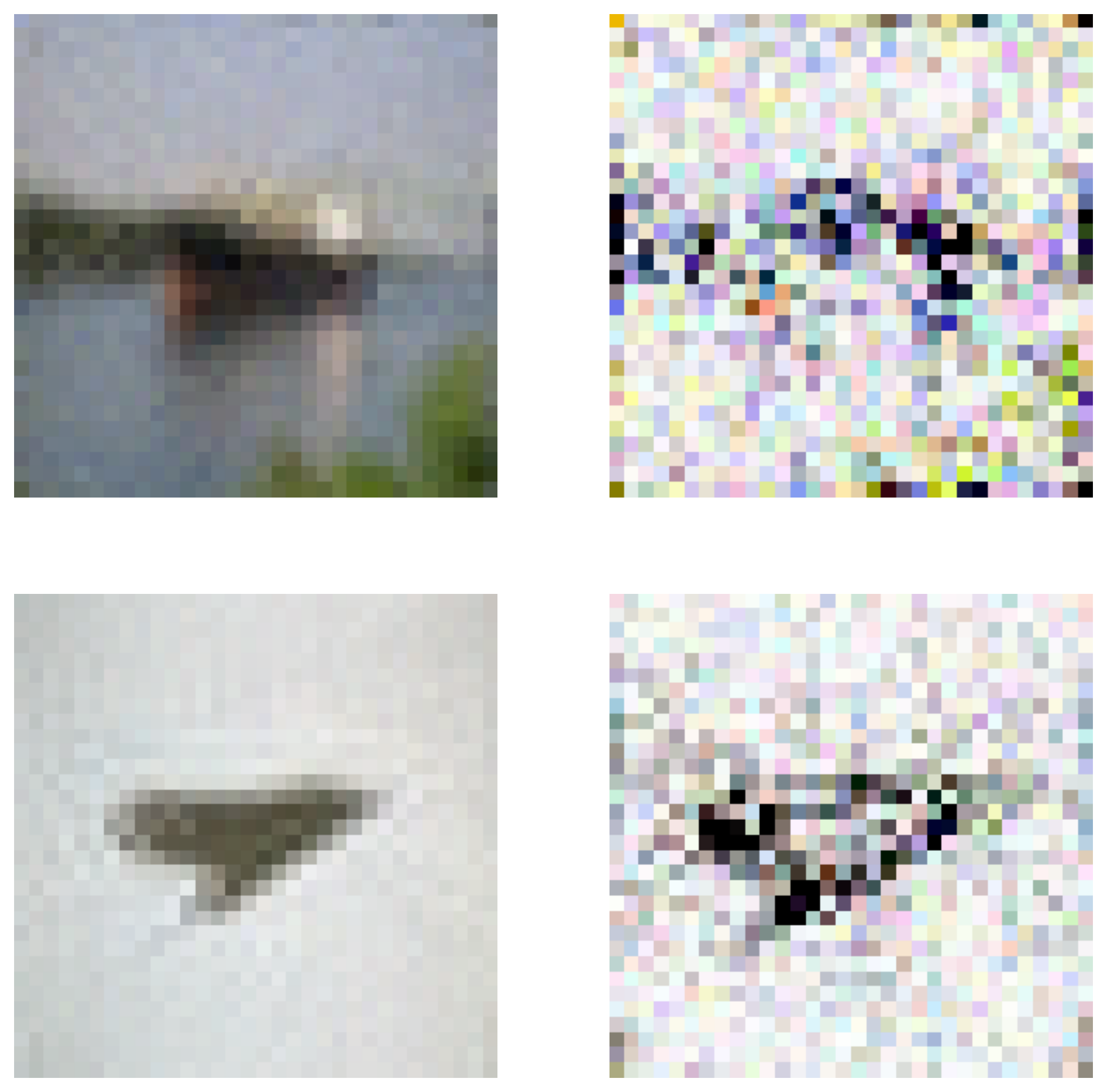}
    \end{tabular}
    \caption{Fourier-$\ell_\infty$ attack}
    \label{fig:image_attack_dftinf_any}
    \end{subfigure}
    \hfill
    \begin{subfigure}[b]{.21\textwidth}
    \begin{tabular}{c}
    $\xx+\ddelta$\hfill
    \,$\ddelta$\hfill
    \\
    \includegraphics[width=.95\textwidth]{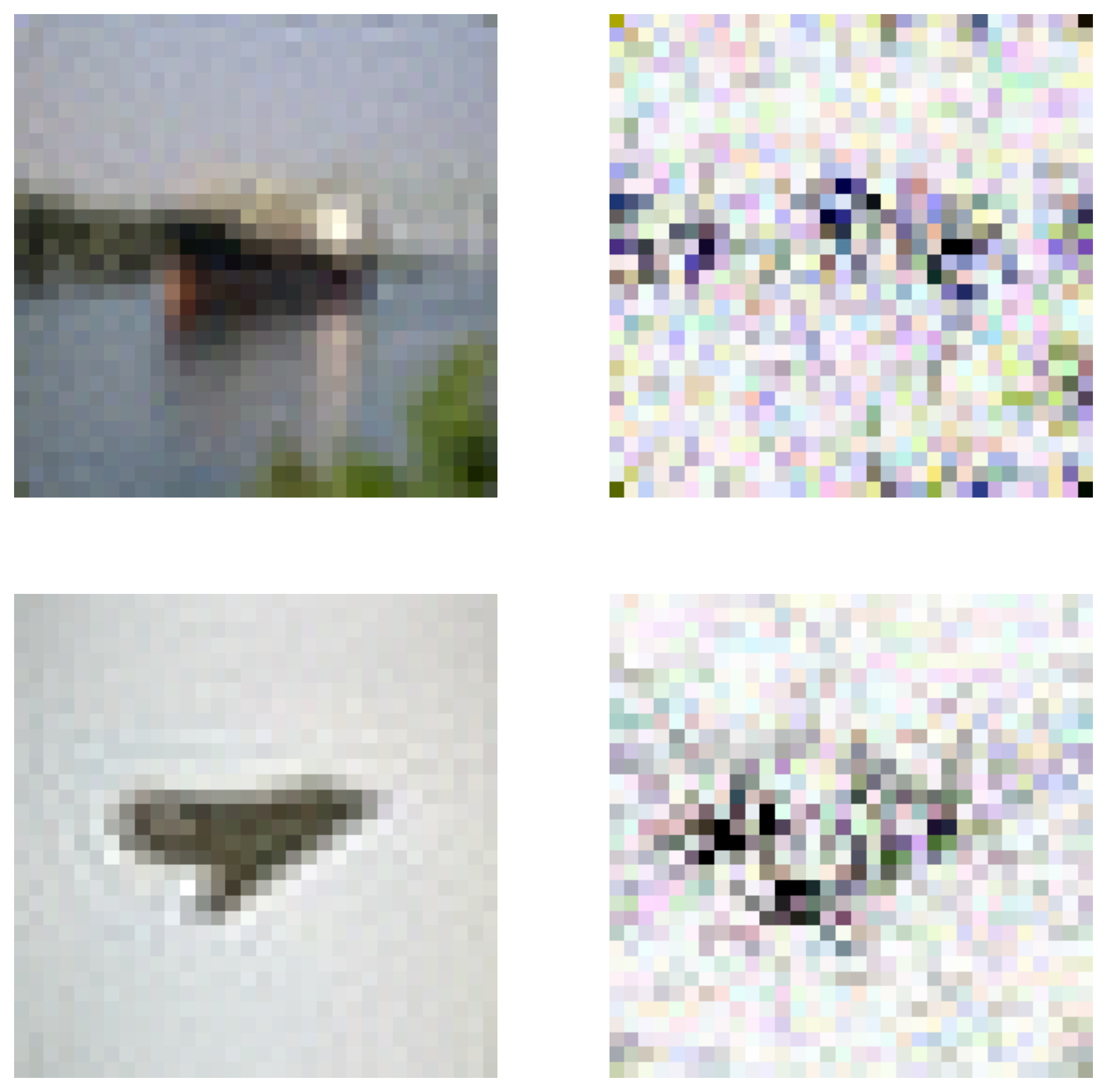}
    \end{tabular}
    \caption{High freq. F-$\ell_\infty$}
    \label{fig:image_attack_dftinf_highf}
    \end{subfigure}
    \hfill
    \begin{subfigure}[b]{.21\textwidth}
    \begin{tabular}{c}
    $\xx+\ddelta$\hfill
    \,$\ddelta$\hfill
    \\
    \includegraphics[width=.95\textwidth]{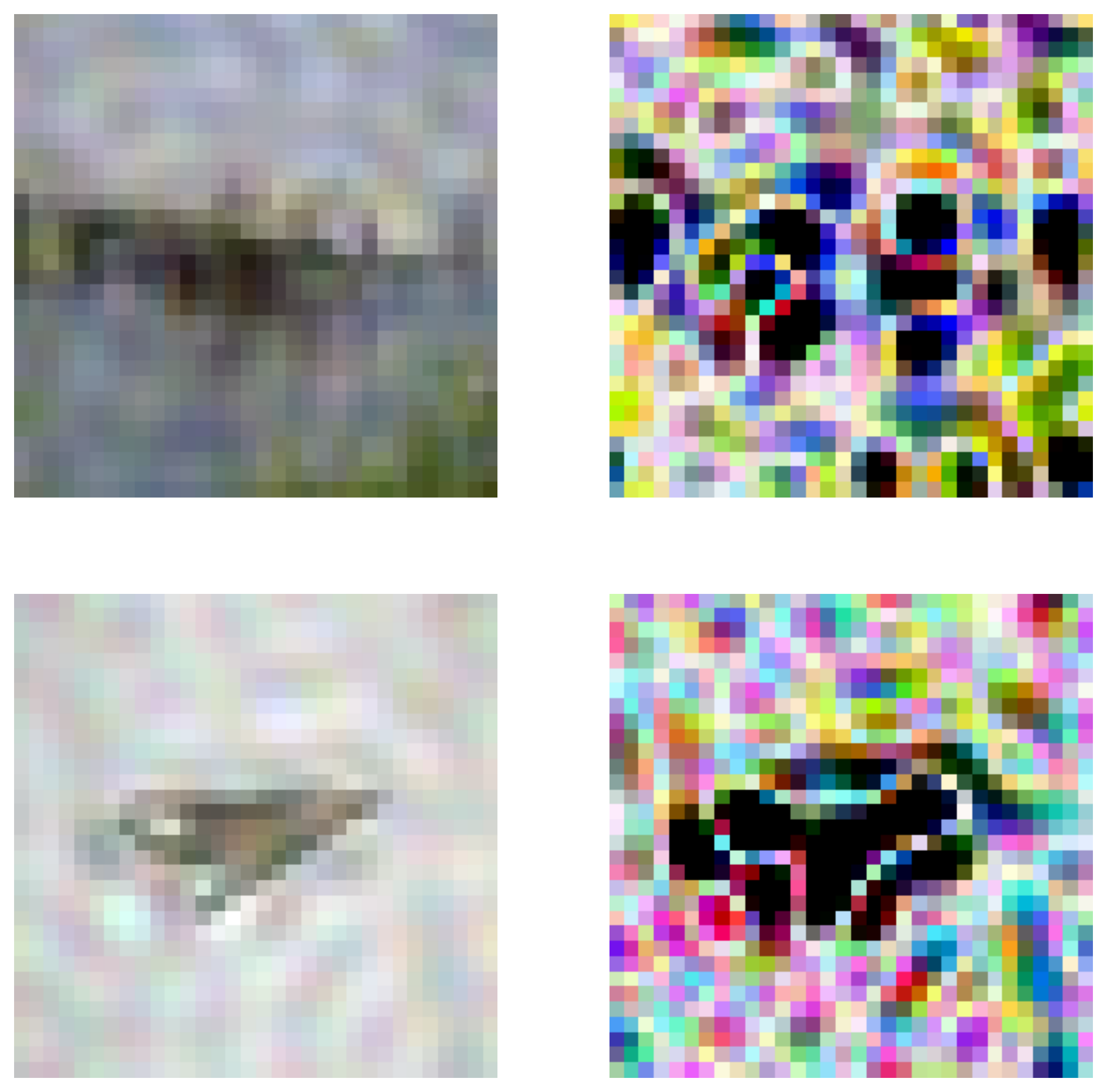}
    \end{tabular}
    \caption{Low Freq. F-$\ell_\infty$}
    \label{fig:image_attack_dftinf_lowf}
    \end{subfigure}
    \hfill
    \vspace*{-4pt}
    \caption{
    \textbf{Adversarial attacks ($\ell_\infty$ and Fourier-$\ell_\infty$)
    against CIFAR-10 classification models.}
    Fourier-$\ell_\infty$ perturbations
    (\subref{fig:image_attack_dftinf_any}) in the spatial
    domain are concentrated around subtle details of the object
    (darker means stronger perturbation).
    In contrast, $\ell_\infty$ perturbations
    (\subref{fig:image_attack_linf})
    are perceived by people as random noise.
    Fourier-$\ell_\infty$ can also be controlled to be
    high or low frequency
    (\subref{fig:image_attack_dftinf_highf},
    \subref{fig:image_attack_dftinf_lowf}). It is more difficult to attack a standard model
    with only low frequency perturbations
    (for all attacks $\varepsilon=8/255$ but for low frequency Fourier-$\ell_\infty$ $\varepsilon=50/255$, otherwise
    attack fails).
    \cref{sec:fourier_vis} shows 
    visualizations for variety of models in RobustBench.
    }
    \label{fig:image_attack}
    \vspace*{-12pt}
\end{figure*}

\section{Explicit Regularization}
\vspace*{-6pt}
\label{sec:explicit_reg}

Above we discussed the impact of optimization method
and model architecture on robustness. Here, 
we discuss explicit regularization as
another choice that affects robustness.

\begin{defn}[Regularized Classification]
\label{def:reg_classifier}
The regularized empirical risk minimization problem
for linear classification
is defined as
$\hat{\ww}(\lambda) =
\argmin_\ww \E_{(\xx,y) \sim \mathcal{D}}
\zeta(y\ww^\top \xx) + \lambda \|\ww\|$,
where $\lambda$ denotes a regularization constant,
$\zeta$ is a monotone loss function,
and $\mathcal{D}$ is a dataset. For simplicity we assume
this problem has a unique solution while the original ERM can
have multiple solutions.
\end{defn}

\begin{thm}[Maximum Margin Classifier using Regularization
(\citet{rosset2004margin}, Theorem 2.1)]
\label{thm:rosset2004_reg}
Consider
linearly separable finite
datasets and monotonically non-increasing
loss functions. Then as $\lambda\rightarrow 0$,
the sequence of solutions, $\hat{\ww}(\lambda)$,
to the regularized problem  in
\cref{def:reg_classifier}, 
converges in direction to a maximum margin classifier
as defined in \eqref{eq:max_margin}. Moreover, if the maximum margin
classifier is unique,
\vspace*{-12pt}
\begin{align}
\lim_{\lambda\rightarrow0}
\frac{\hat{\ww}(\lambda)}{\|\hat{\ww}(\lambda)\|}
&= \argmax_{\ww : \|\ww\| \leq 1}
\min_i y_i \ww^\top \xx_i\, .
\end{align}
\end{thm}
\vspace*{-7pt}
The original proof in \citep{rosset2004margin}
was given specifically for $\ell_p$
norms, however we observe that their proof only requires
convexity of the norm, so we state it more generally.
Quasi-norms such as $\ell_p$
for $p<1$ are not covered by this theorem. In addition,
the condition on the loss function is weaker than
our strict monotonic decreasing condition as shown
in \citep[Appendix A]{nacson2019convergence}.

We use this result to derive our \cref{cor:reg_is_max_margin}.

\begin{cor}[Maximally Robust Classifier via
Infinitesimal Regularization]
\label{cor:reg_is_max_margin}
For linearly separable data, under conditions of
\cref{thm:rosset2004_reg},
the sequence of solutions to regularized classification problems
converges in direction to a maximally robust classifier. That is,
$\lim_{\lambda\rightarrow 0}\hat{\ww}(\lambda)/\|\hat{\ww}(\lambda)\|$
converges to a solution of
$\argmax_{\ww}
\{\varepsilon\,|\,y_i \ww^\top (\xx_i + \ddelta) > 0 
,\,~\forall i,\,\|\ddelta\|_\ast \leq \varepsilon\}$.
\vspace*{-12pt}
\end{cor}
\begin{proof}
\vspace*{-2pt}
By \cref{thm:rosset2004_reg},
the margin of the sequence of regularized classifiers,
$\min_i y_i \frac{\hat{\ww}(\lambda)^\top}{\|\hat{\ww}(\lambda)\|} \xx_i$,
converges to the maximum margin,
$\max_{\ww : \|\ww\| \leq 1} \min_i y_i \ww^\top \xx_i$.
By \cref{lem:max_robust_to_min_norm_linear},
any maximum margin classifier w.r.t. $\|\cdot\|$
gives a maximally robust classifier w.r.t. $\|\cdot\|_\ast$.
\end{proof}
\vspace*{-9pt}
Assuming the solution to the
regularized problem is unique,
the regularization term replaces other implicit
biases in minimizing the empirical risk.
The regularization coefficient
controls the trade-off between robustness
and standard accuracy.
The advantage of this formulation compared with
adversarial training is that we do not need the knowledge
of the maximally robust $\varepsilon$ to find
a maximally robust classifier. It suffices to choose
an infinitesimal regularization coefficient.
\citet[Theorem 4.1]{wei2019regularization}
generalized \cref{thm:rosset2004_reg} for a family of
classifiers that includes fully-connected networks with ReLU
non-linearities, which allows for potential extension of
\cref{cor:reg_is_max_margin} to non-linear models.
There remain gaps in this extension
(see \cref{sec:nonlinear}).

Explicit regularization has been explored
as an alternative approach to
adversarial training~\citep{hein2017formal,
sokolic2017robust,
zhang2019theoretically,
qin2019adversarial,
guo2020connections}.
To be clear, we do not propose a new regularization method but rather,
we provide a framework for deriving and guaranteeing
the robustness of existing and future regularization methods.

\vspace*{-7pt}

\section{Experiments}
\label{sec:experiments}
\vspace*{-7pt}

\begin{figure*}[t]
\vspace*{-10pt}
\centering
\begin{subfigure}[c]{0.26\textwidth}
\includegraphics[width=\textwidth]{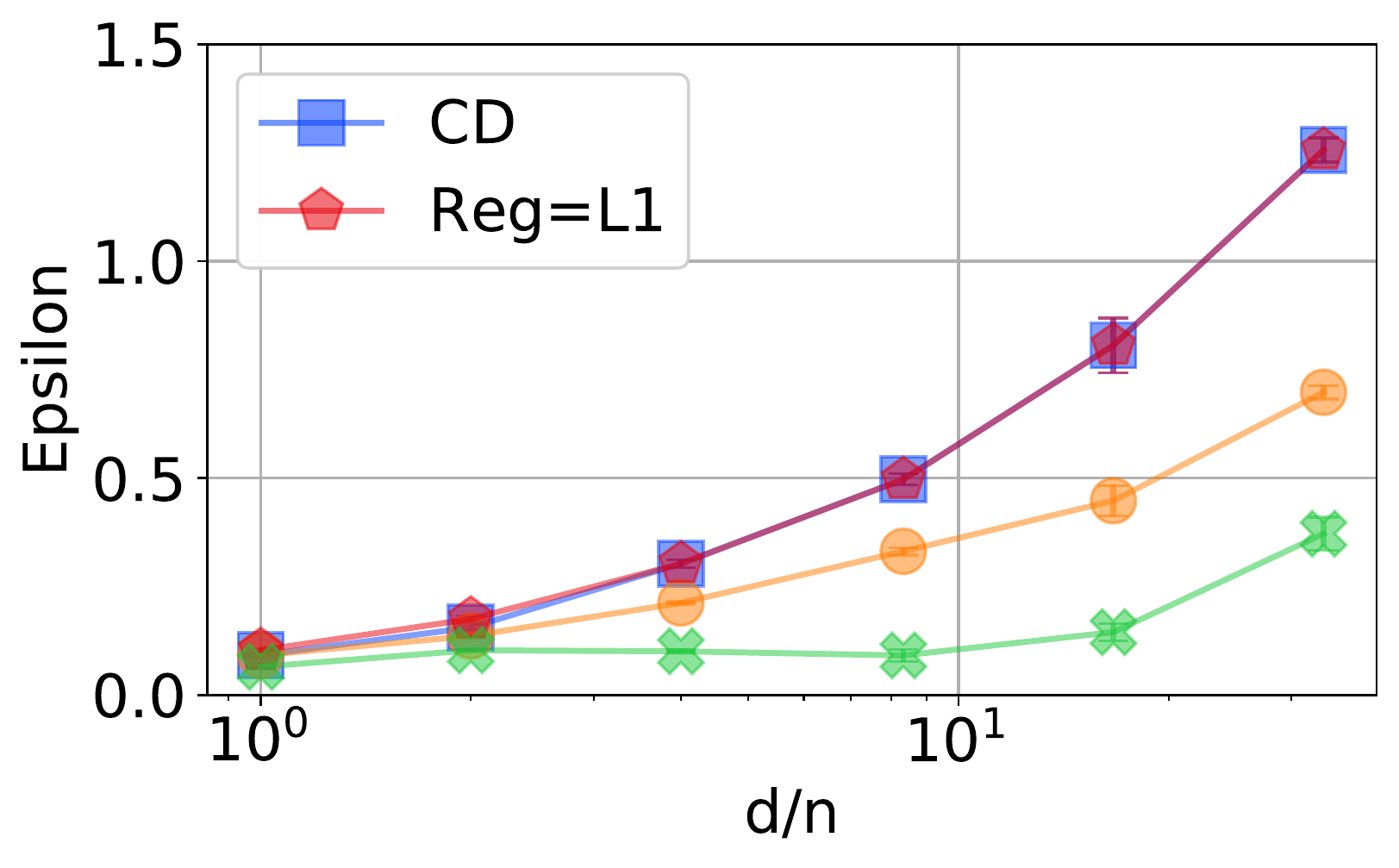}
\caption{$\ell_\infty$ attack}
\label{fig:linear_max_eps_linf}
\end{subfigure}
\hspace{1pt}
\begin{subfigure}[c]{0.26\textwidth}
\includegraphics[width=\textwidth]{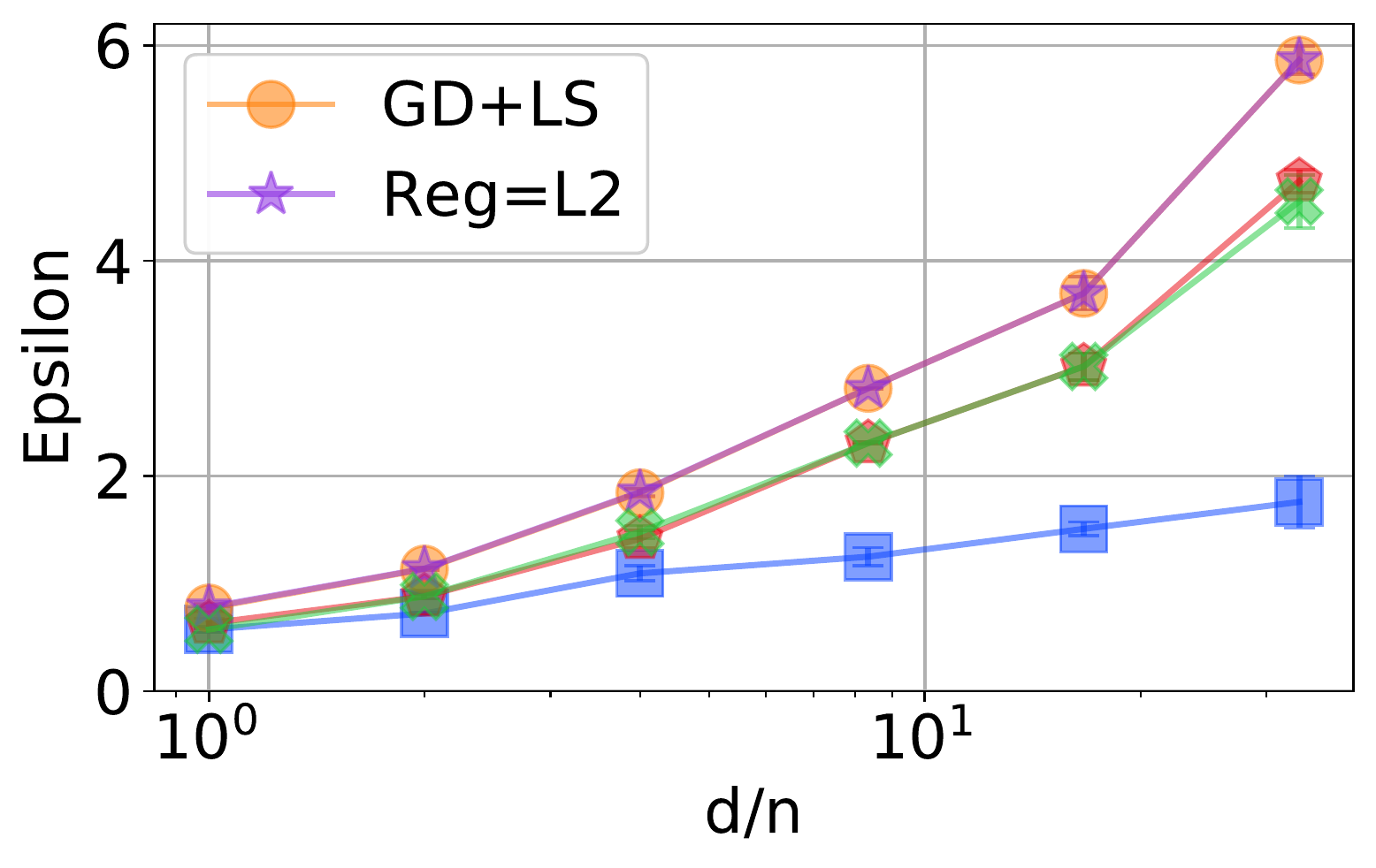}
\caption{$\ell_2$ attack}
\label{fig:linear_max_eps_l2}
\end{subfigure}
\hspace{1pt}
\begin{subfigure}[c]{0.26\textwidth}
\includegraphics[width=\textwidth]{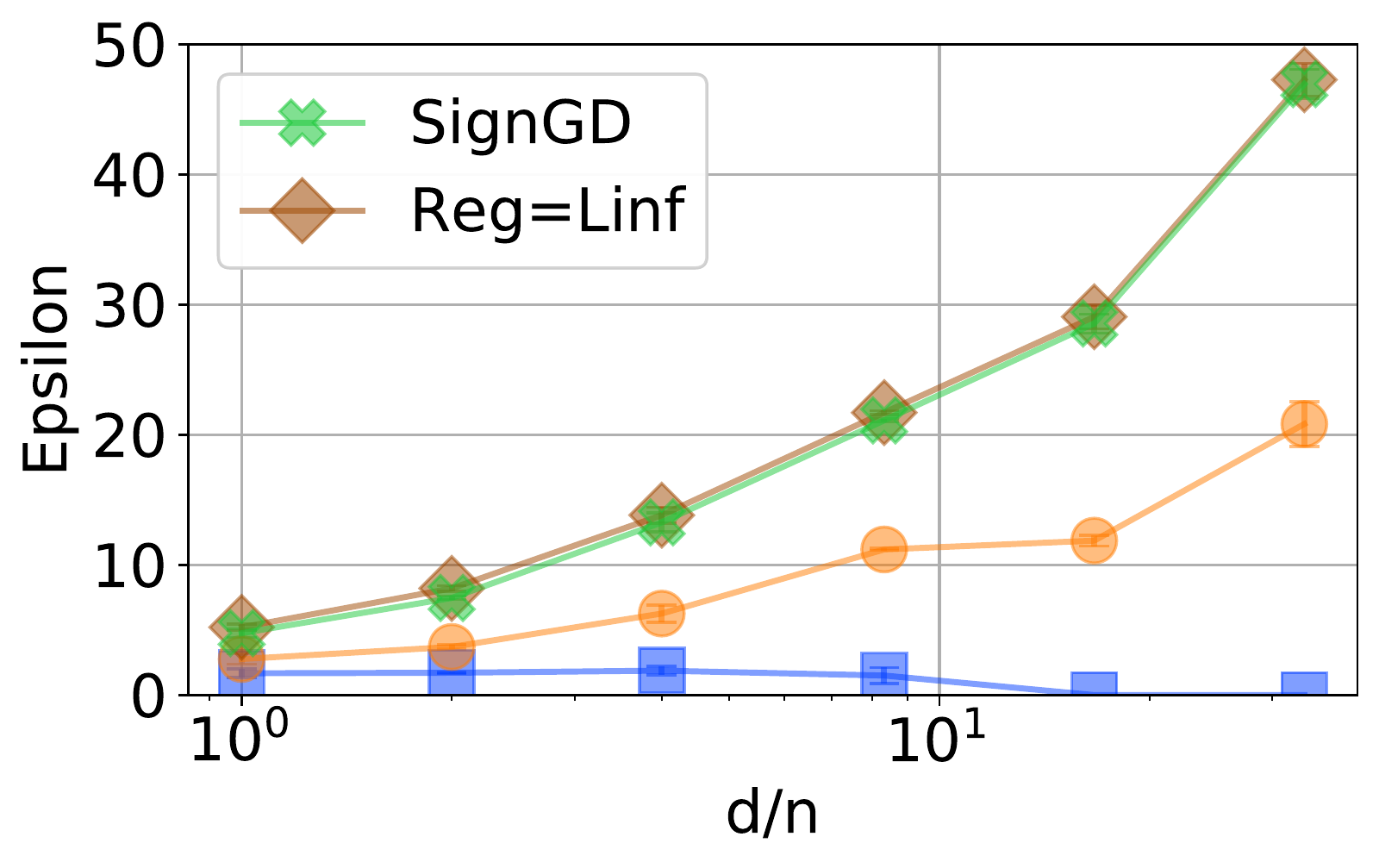}
\caption{$\ell_1$ attack}
\label{fig:linear_max_eps_l1}
\end{subfigure}
\hspace{1pt}
\begin{subfigure}[c]{0.12\textwidth}
\vspace*{-25pt}
\fbox{
\includegraphics[width=\textwidth]{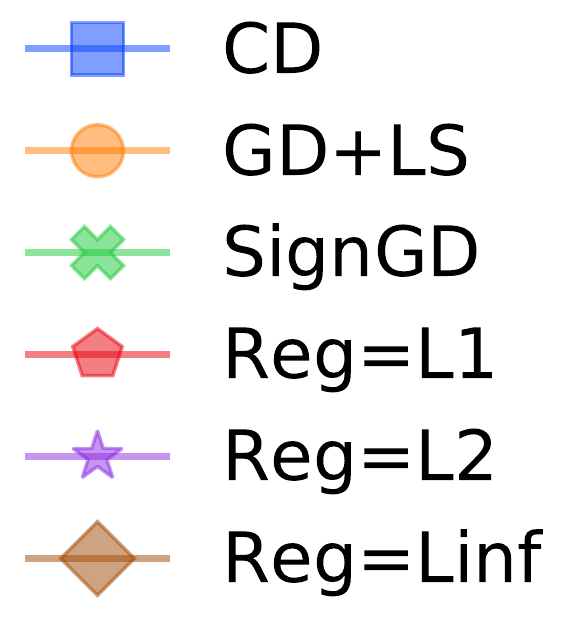}}
\end{subfigure}
\vspace*{-5pt}
\caption{\textbf{Maximally robust perturbation size ($\varepsilon$)
for linear models against $\ell_\infty$, $\ell_2$, and $\ell_1$ attacks.}
For each attack, there exists one optimizer and one regularization method
that finds a maximally robust classifier
(inner legends). 
We compare Coordinate Descent (CD), Gradient Descent with Line Search
(GD+LS), Sign Gradient Descent (SignGD), and explicit $\ell_1$, $\ell_2$,
and $\ell_\infty$ regularization.
The gap between methods grows with the overparametrization ratio ($d/n$).
(More figures in \cref{sec:margin})}
\label{fig:linear_max_eps}
    \vspace*{-10pt}
\end{figure*}

This section empirically compares approaches to finding
maximally robust classifiers.
\cref{sec:cifar10} evaluates the robustness of
CIFAR-10~\citep{krizhevsky2009learning}
image classifiers against our Fourier-$\ell_\infty$ attack. We
implement our attack in AutoAttack~\citep{croce2020reliable}
and  evaluate the robustness
of recent defenses available in
RobustBench~\citep{croce2020robustbench}.
Details of the experiments and additional
visualizations are in \cref{sec:exp_ext}.

\subsection{Maximally Robust to
\texorpdfstring{$\ell_\infty,\ell_2,\ell_1$}{Linf, L2, L1}, and Fourier-\texorpdfstring{$\ell_\infty$}{Linf} Bounded Attacks}
\vspace*{-5pt}

\begin{wrapfigure}[12]{r}{0.5\textwidth}
\vspace*{-12pt}
\centering
\begin{subfigure}[l]{.6\linewidth}
\includegraphics[width=\textwidth]{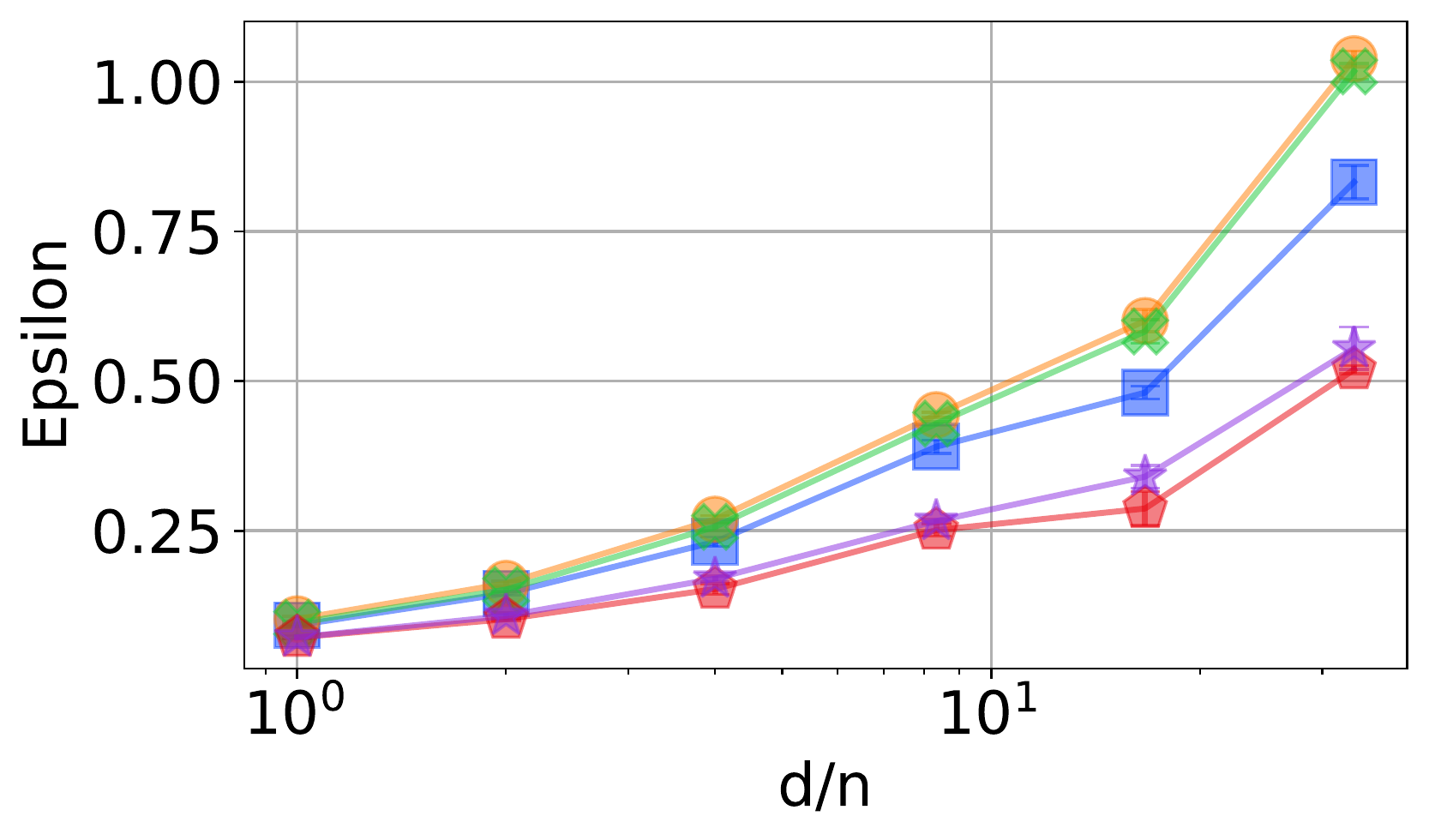}
\end{subfigure}
\begin{subfigure}[l]{.35\linewidth}
\vspace*{-25pt}
\fbox{
\includegraphics[width=\textwidth]{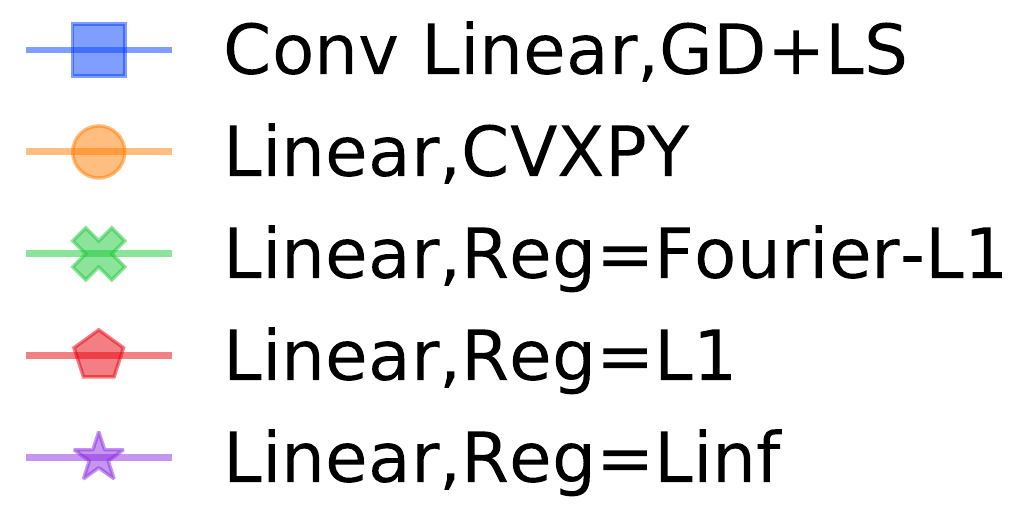}}
\end{subfigure}
\vspace*{-0.3cm}
\caption{\textbf{Maximally robust $\varepsilon$ against
Fourier-$\ell_\infty$ attack.}
Explicit Fourier-$\ell_1$ regularization finds
a maximally robust classifier as it achieves similar robustness
as CVXPY's solution.
A linear convolutional model converges to a solution
more slowly.}
\label{fig:conv}
\vspace*{-5pt}
\end{wrapfigure}

Figs.\ \ref{fig:linear_max_eps} and \ref{fig:conv} plot the maximally robust $\epsilon$ as a function of the overparametrization ratio $d/n$, where $d$ is the model dimension and $n$ is the number of data points.
\cref{fig:linear_max_eps} shows robustness against $\ell_\infty$, $\ell_2$, and $\ell_1$ attacks for linear models.
Coordinate descent and explicit $\ell_1$ regularization
find a maximally robust $\ell_\infty$ classifier.
Coordinate descent and $\ell_2$ regularization
find a maximally robust $\ell_2$ classifier.
Sign gradient descent and $\ell_\infty$ regularization
find a maximally robust $\ell_1$ classifier.
The gap between margins grows as $d/n$ increases.
\cref{fig:conv} shows robustness against
Fourier-$\ell_\infty$ attack; training a 2-layer
linear convnet with gradient descent converges to a maximally robust classifier, albeit slowly.

For these plots we synthesized linearly separable data focusing on overparametrized classification problems (i.e., $ d > n $).
Plotting the overparametrization ratio shows how robustness
changes as models become more complex.
We compare models by computing the maximal $\varepsilon$
against which they are robust, or equivalently,
the margin for linear models, $\min_i y_i\ww^\top\xx_i/\|\ww\|$.
As an alternative to the margin, we estimate the maximal $\varepsilon$
for a model by choosing a range of potential values,
generating adversarial samples, and finding the largest value against which the classification error is zero.
Generating adversarial samples involves optimization, and requires
more implementation detail compared with computing the margin.
Plots in this section are based on generating adversarial
samples to match common practice in the evaluation of non-linear models.
Matching margin plots are presented in \cref{sec:margin},
which compare against the solution found using CVXPY~\citep{diamond2016cvxpy}
and adversarial training given the maximal $\varepsilon$.
Our plots depict mean and error bars for $3$ random seeds.

\vspace*{-7pt}
\subsection{Plotting the Trade-offs}
\label{sec:tradeoffs}
\vspace*{-7pt}
\begin{figure*}[t]
\vspace*{-0.15cm}
\centering
\begin{subfigure}[b]{0.24\textwidth}
\includegraphics[width=\textwidth]{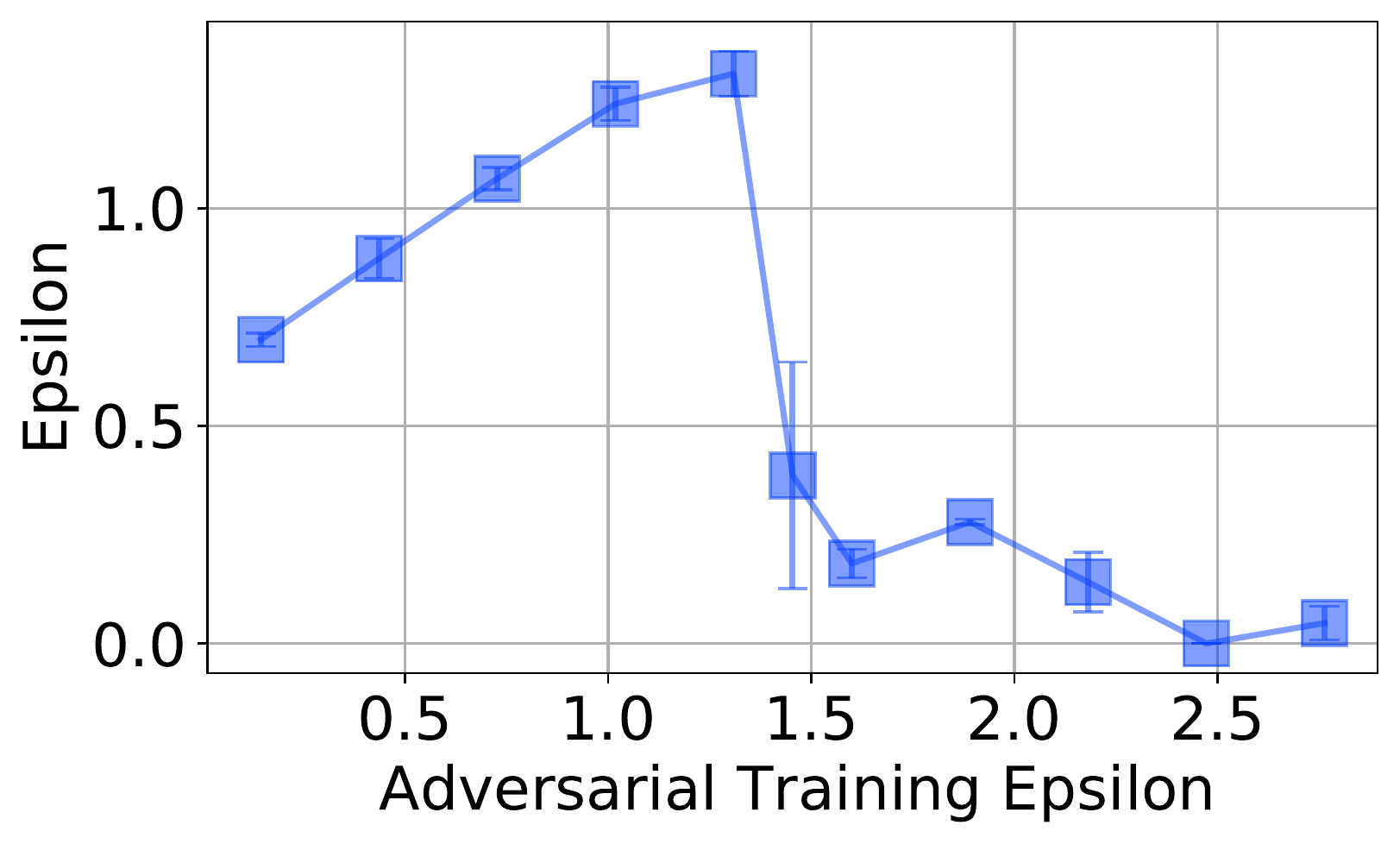}
\caption{Adversarial training}
\label{fig:tradeoffs_at}
\end{subfigure}
\begin{subfigure}[b]{0.24\textwidth}
\includegraphics[width=\textwidth]{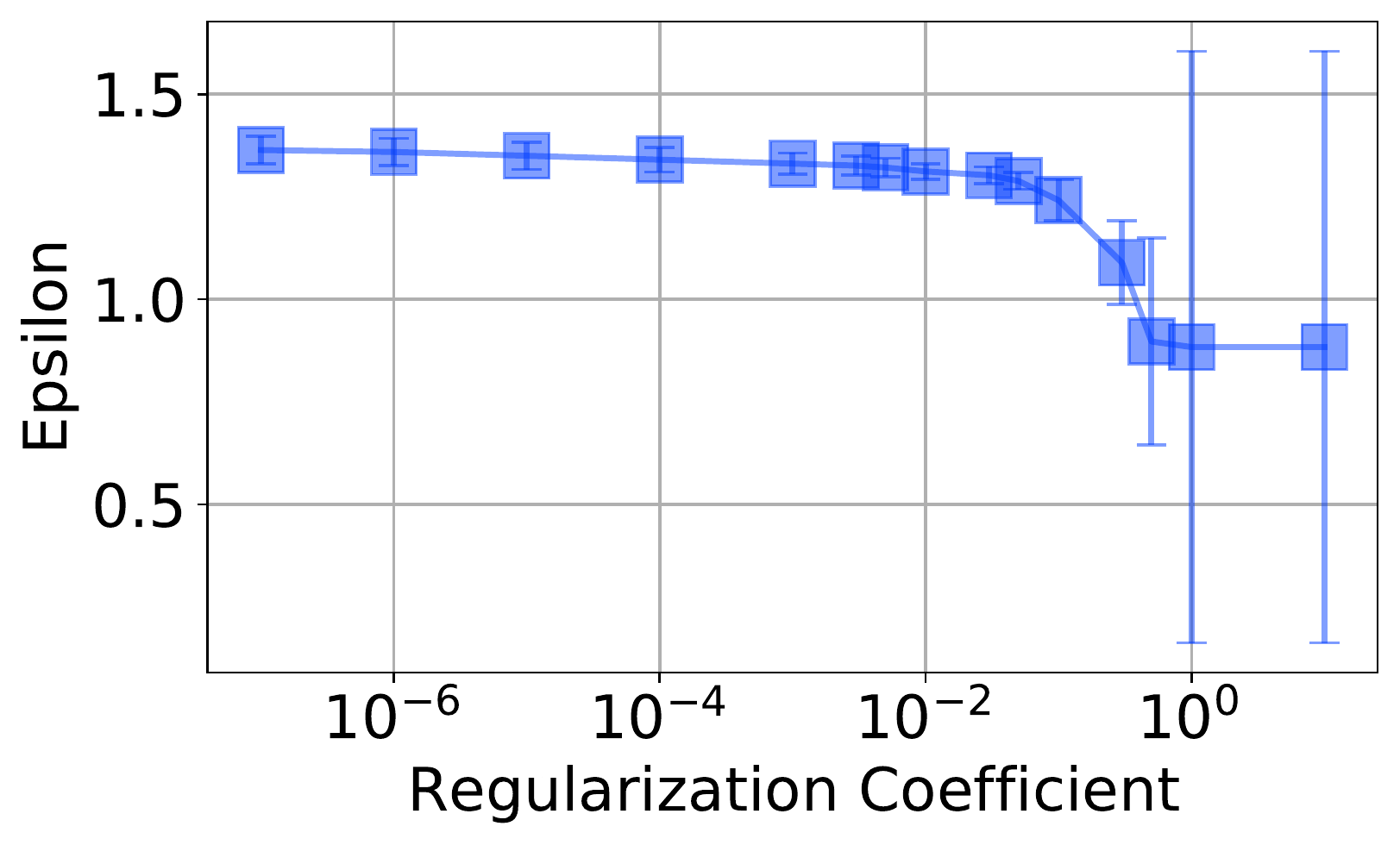}
\caption{$\ell_1$ regularization}
\label{fig:tradeoffs_l1}
\end{subfigure}
\begin{subfigure}[b]{0.24\textwidth}
\includegraphics[width=\textwidth]{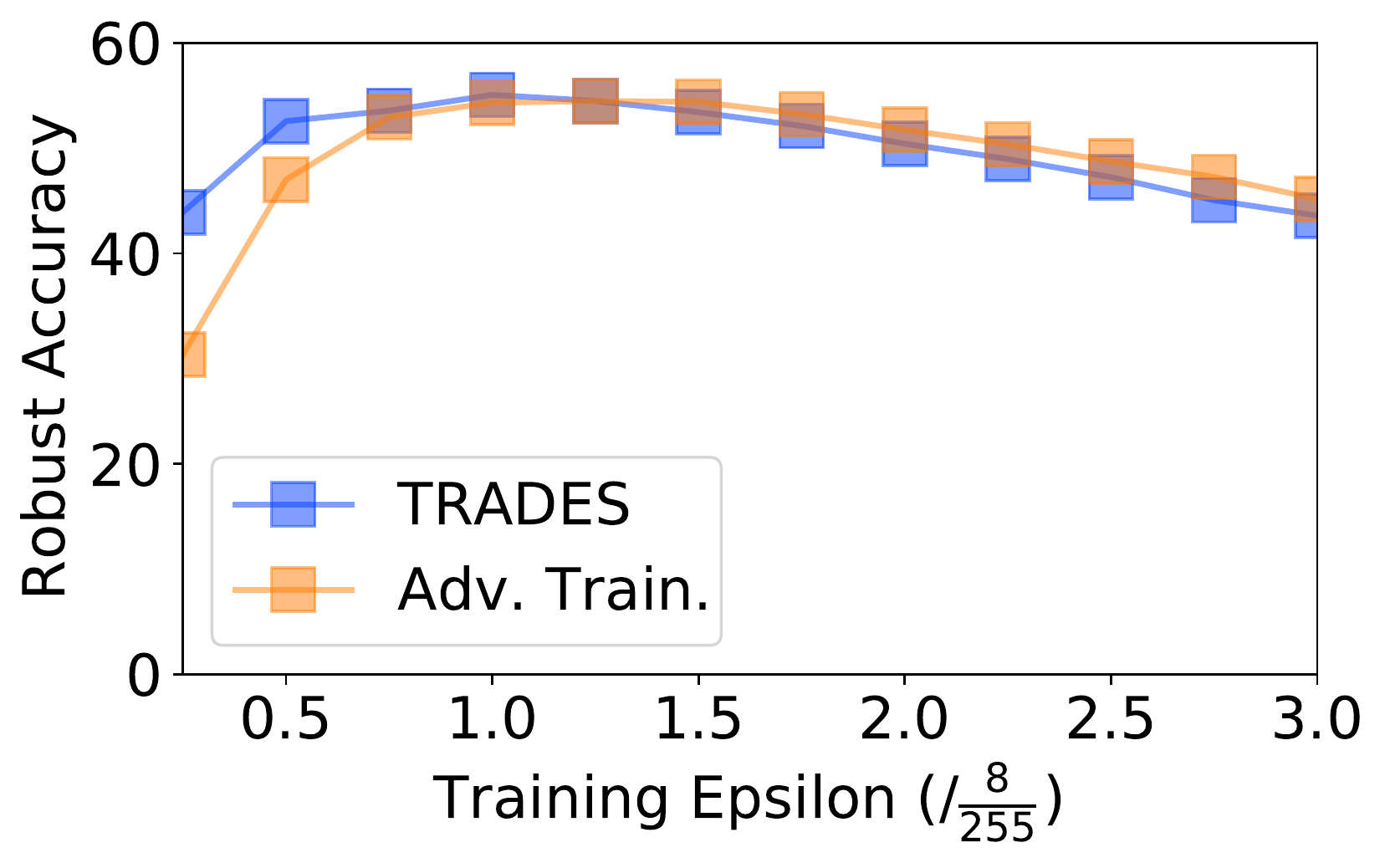}
\caption{CIFAR-10 Train. Eps.}
\label{fig:tradeoffs_cifar10_eps}
\end{subfigure}
\begin{subfigure}[b]{0.24\textwidth}
\includegraphics[width=\textwidth]{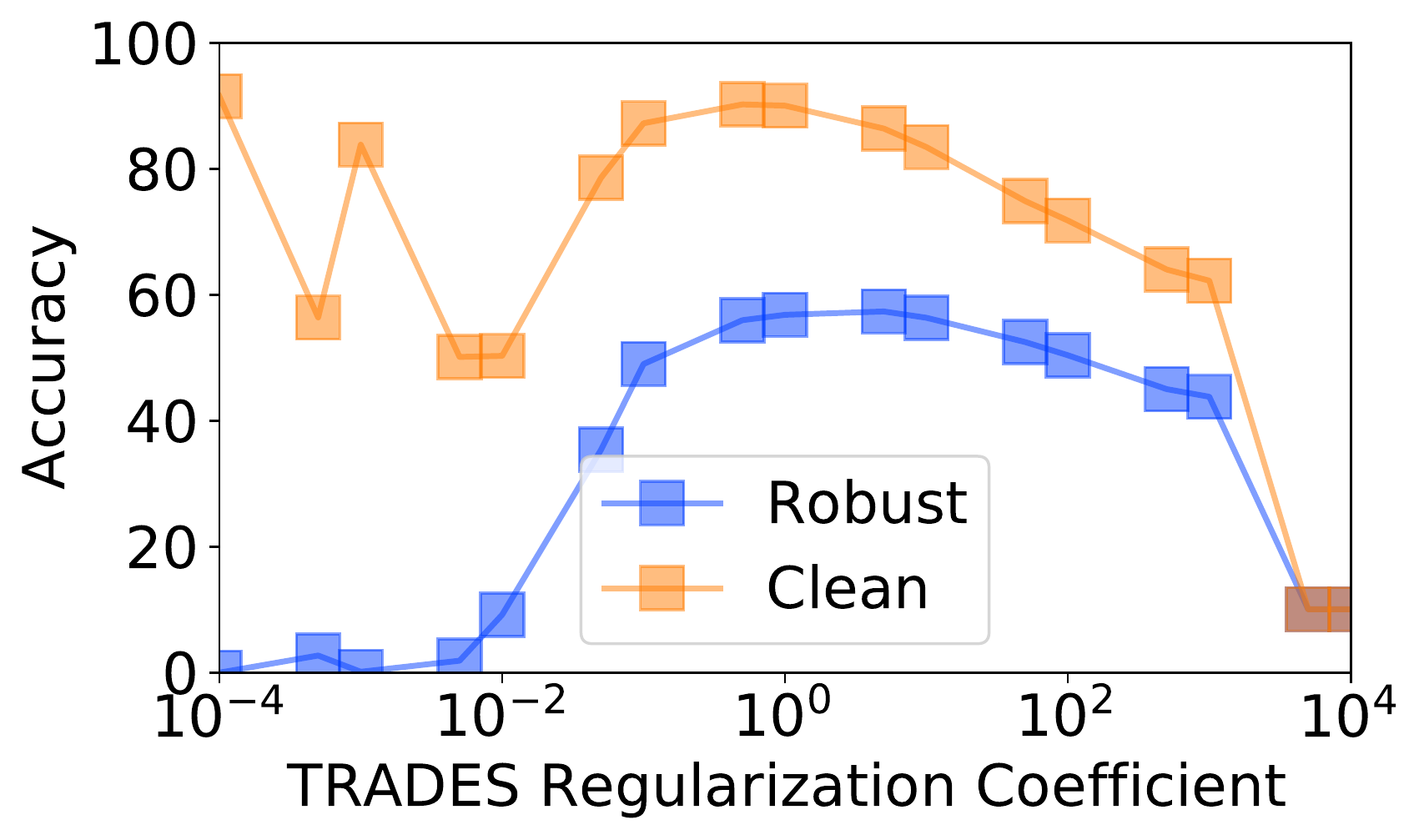}
\caption{CIFAR-10 TRADES}
\label{fig:tradeoffs_cifar10_trades}
\end{subfigure}
\vspace*{-0.1cm}
\caption{\textbf{Trade-off
in robustness against $\ell_\infty$ attack
in linear models and CIFAR-10.}
We plot the maximally robust $\varepsilon$ for adversarial
training and explicit regularization.
Robustness is controlled using $\varepsilon$ in adversarial training (\subref{fig:tradeoffs_at})
and regularization coefficient in explicit regularization (\subref{fig:tradeoffs_l1}).
Using adversarial training we have to search for the maximal
$\varepsilon$ but for explicit regularization it suffices
to choose a small regularization coefficient.
Similarly, on CIFAR-10,
the highest robustness at a fixed test $\varepsilon$
is achieved for $\varepsilon$ used during
training (\subref{fig:tradeoffs_cifar10_eps}).
In contrast to optimal linear regularizations,
TRADES shows degradation as the
regularization coefficient decreases
(\subref{fig:tradeoffs_cifar10_trades}).
Discussion in \cref{sec:tradeoffs}
}
\label{fig:tradeoffs_linf}
    \vspace*{-16pt}
\end{figure*}

\cref{fig:tradeoffs_linf} illustrates the trade-off between
standard accuracy and adversarial robustness.
Adversarial training finds the maximally robust classifier
only if it is trained with the
knowledge of the maximally robust $\varepsilon$ (\cref{fig:tradeoffs_at}).
Without this knowledge, we have to search for the maximal $\varepsilon$
by training multiple models.
This adds further computational complexity to adversarial
training which performs an alternated optimization.
In contrast, explicit regularization converges to a
maximally robust classifier for a small enough regularization constant (\cref{fig:tradeoffs_l1}).

On CIFAR-10, we compare adversarial training with
the regularization method TRADEs~\citep{zhang2019theoretically}
following the state-of-the-art
best practices~\citep{gowal2020uncovering}.
Both methods depend on a constant $\varepsilon$
during training.
\cref{fig:tradeoffs_cifar10_eps} shows optimal robustness
is achieved for a model trained and tested with
the same $\varepsilon$. When the test
$\varepsilon$ is unknown, both methods need to search for
the optimal $\varepsilon$.
Given the optimal training $\varepsilon$,
\cref{fig:tradeoffs_cifar10_trades} investigates whether
TRADES performs similar to an optimal linear regularization
(observed in \cref{fig:tradeoffs_l1}),
that is the optimal robustness is achieved with infinitesimal
regularization.
In contrast to the linear regime, the robustness degrades with
smaller regularization.
We hypothesize that with enough model capacity,
using the optimal $\varepsilon$,
and sufficient training iterations, smaller
regularization should improve robustness.
That suggests that there is potential for improvement
in TRADES and better understanding of robustness in non-linear
models.

\vspace*{-7pt}
\subsection{CIFAR-10 Fourier-\texorpdfstring{$\ell_\infty$}{Linf} Robustness}
\label{sec:cifar10}
\vspace*{-8pt}

\begin{wrapfigure}[13]{r}{0.42\linewidth}
\vspace*{-27pt}
    \centering
    \includegraphics[width=\linewidth]{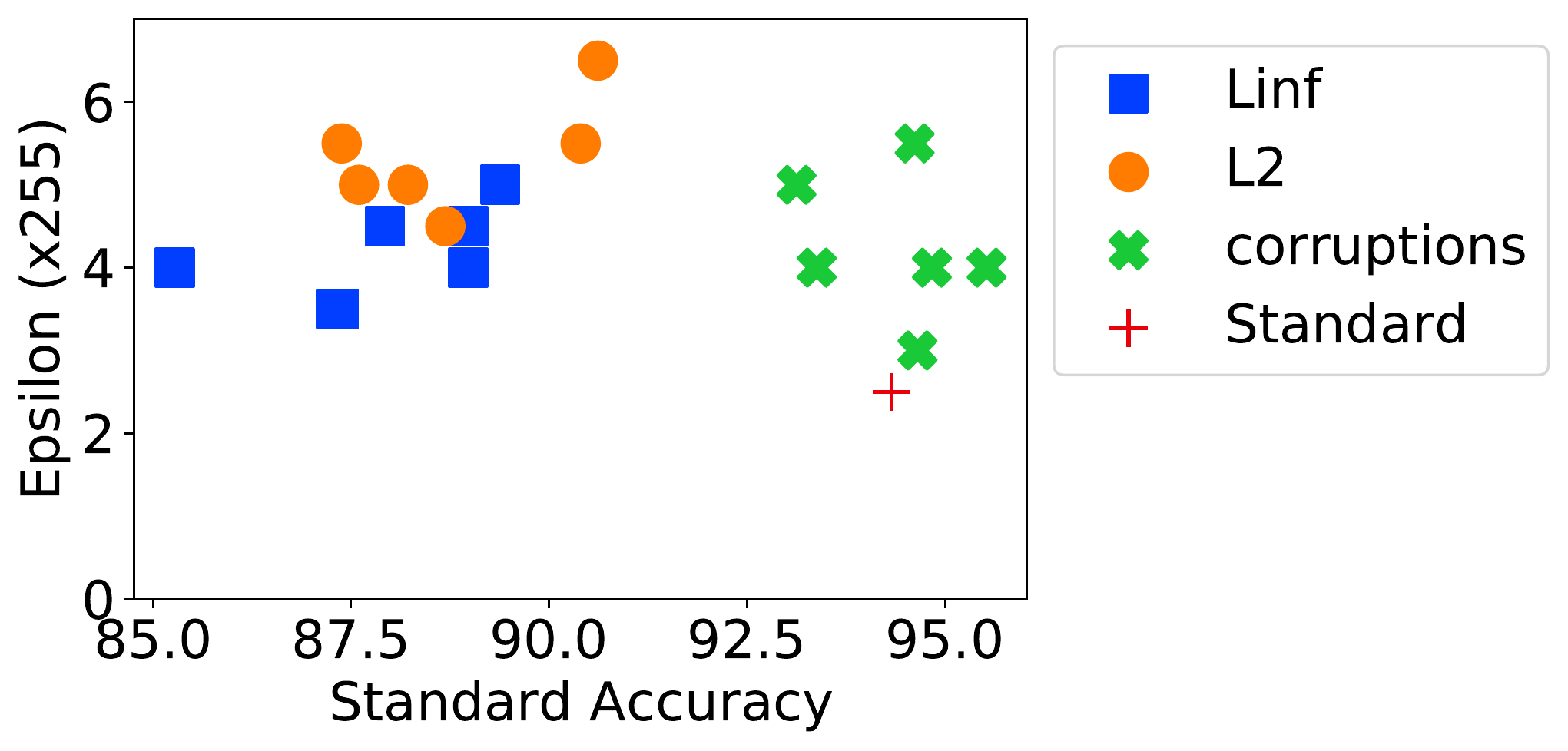}
    \caption{\textbf{Maximally robust
$\varepsilon$ against Fourier-$\ell_\infty$ for recent defenses.}
Color and shape denote the type of
robust training: adversarial ($\ell_2$  or $\ell_\infty$),
corruptions,
or standard training.
Fourier-$\ell_\infty$ is a strong attack against
current robust models.
}
    \label{fig:cifar10_acc}
\end{wrapfigure}

\cref{fig:cifar10_acc} reports the maximally robust $\varepsilon$
of image classification models on CIFAR-10.
We evaluate top defenses on the leaderboard of
RobustBench~\citep{croce2020robustbench}.
The attack methods are APGD-CE and APGD-DLR with default
hyperparameters in RobustBench and $\varepsilon=8/255$.
Theoretical results do not provide guarantees beyond the maximally
robust $\varepsilon$.
Even robust models
against corruptions with no adversarial training achieve
similar robustness to $\ell_2$/$\ell_\infty$ models.
The maximal $\varepsilon$ is the largest
at which adversarial accuracy is no more than $1\%$ worse
than the standard accuracy.
All models have almost zero accuracy
against larger, but still perceptually small, perturbations
($\varepsilon=20/255$).
\cref{sec:fourier_vis} gives more examples of
Fourier-$\ell_\infty$ attacks and band-limited variations
similar to \cref{fig:image_attack} for robustly trained models, showing that
perturbations are qualitatively different from those against
the standard model.

\vspace*{-10pt}
\section{Related work}
\label{sec:related}
\vspace*{-10pt}

This paper bridges two bodies of work on adversarial robustness and
optimization bias. 
As such there are many related works,
the most relevant of which we discuss here.
Prior works either did not connect optimization bias
to adversarial robustness beyond
margin-maximization~\citep{ma2020increasing,ding2018max,elsayed2018large}
or only considered adversarial
training with a given perturbation size~\citep{li2019implicit}.

\vspace*{-8pt}
\paragraph{Robustness Trade-offs}
Most prior work defines the metric for robustness and generalization
using an expectation over the loss. Instead, we define robustness
as a set of classification constraints.
Our approach better matches the security perspective that
even a single inaccurate prediction is a vulnerability.
The limitation is explicit constraints
only ensure perfect accuracy near the training set.
Standard generalization remains to be studied using
other approaches such with assumptions on the data distribution.
Existing work has used assumptions about the data distribution to
achieve explicit trade-offs between robustness and standard generalization
~\citep{dobriban2020provable,
javanmard2020preciseb,%
javanmard2020precise,
raghunathan2020understanding,
tsipras2018robustness,
zhang2019theoretically,
schmidt2018adversarially,
fawzi2018adversarial,
fawzi2018analysis}.

\vspace*{-8pt}
\paragraph{Fourier Analysis of Robustness.}
Various observations have been made about Fourier properties
of adversarial perturbations against deep non-linear models~\citep{ilyas2019adversarial,
tsuzuku2019structural,
sharma2019effectiveness}.
\citet{yin2019fourier} showed that adversarial
training increases robustness to perturbations concentrated
at high frequencies and reduces robustness to perturbations
concentrated at low frequencies. \citet{ortiz2020hold}
also observed that the measured margin of classifiers
at high frequencies is larger than the margin at
low frequencies.
Our \cref{cor:max_robust_to_min_norm_conv}
does not distinguish between
low and high frequencies but we establish an exact characterization of
robustness.
\citet{caro2020using} hypothesized about
the implicit robustness to $\ell_1$ perturbations in the Fourier
domain while we prove maximal robustness to
Fourier-$\ell_\infty$ perturbations.

\vspace*{-8pt}
\paragraph{Architectural Robustness.}
An implication of our results is that robustness can be achieved
at a lower computational cost compared with adversarial training by
various architectural choices as recently explored~\citep{xie2020smooth,
galloway2019batch,
awais2020towards}.
Moreover, for architectural choices that align with human biases, standard
generalization can also improve~\citep{vasconcelos2020effective}.
Another potential future direction is to rethink $\ell_p$ robustness as an architectural
bias and find inspiration in human visual system for
appropriate architectural choices.

\vspace*{-8pt}
\paragraph{Robust Optimization}
A robust counterpart to an optimization
problem considers uncertainty in the data and optimizes for the worst-case.
\citet{ben2009robust}
provided extensive formulations and discussions
on robust counterparts
to various convex optimization problems. Adversarial robustness
is one such robust counterpart and many other robust counterparts
could also be considered in deep learning. An example is
adversarial perturbations with different norm-ball constraints
at different training inputs.
\citet{madry2017towards}
observed the link between robust optimization
and adversarial robustness where the objective
is a min-max problem
that minimizes the worst-case loss. However, they did not
consider the more challenging problem of maximally robust
optimization that we revisit.

\vspace*{-8pt}
\paragraph{Implicit bias of optimization methods.}
Minimizing the empirical risk for
an overparametrized model with more parameters than the training
data has multiple solutions.
\citet{zhang2017understanding} observed that overparametrized
deep models can even fit to randomly labeled training data, yet
given correct labels they consistently generalize to test data.
This behavior has been explained using the implicit bias of
optimization methods towards particular solutions.
\citet{gunasekar2018characterizing} proved that
minimizing the empirical risk using  steepest descent
and mirror descent have an implicit bias towards minimum norm
solutions in overparametrized linear classification. Characterizing
the implicit bias in linear regression proved to be more challenging
and dependent on the initialization.
\citet{ji2018gradient} proved that
training a deep linear classifier using gradient descent
not only implicitly converges to the minimum norm classifier
in the space of the product of parameters,
each layer is also biased towards rank-$1$ matrices aligned with adjacent layers.
\citet{gunasekar2018implicit} proved the implicit bias of
gradient descent in training linear convolutional classifiers
is towards minimum norm solutions in the Fourier domain that
depends on the number of layers.
\citet{ji2020directional} has established the directional alignment
in the training of deep linear networks using gradient flow as well as
the implicit bias of training deep 2-homogeneous networks.
In the case of gradient flow (gradient descent with infinitesimal
step size) the implicit bias of training multi-layer linear models
is towards rank-$1$ layers that satisfy directional alignment
with adjacent layers~\citep[Proposition 4.4]{ji2020directional}.
Recently, \citet{yun2020unifying} has proposed a unified framework
for implicit bias of neural networks using tensor formulation 
that includes fully-connected, diagonal, and convolutional
networks and weakened the convergence assumptions.

Recent theory of generalization in deep learning, in particular
the double descent phenomenon, studies the generalization
properties of minimum norm solutions
for finite and noisy training sets~\citep{hastie2019surprises}.
Characterization of the double descent phenomenon relies on the
implicit bias of optimization methods
while using additional assumptions
about the data distribution.
In contrast, our results only rely on the implicit bias of optimization
and hence are independent of the data distribution.

\vspace*{-8pt}
\paragraph{Hypotheses.}
\citet{goodfellow2014explaining}  proposed the linearity
hypothesis that informally suggests
$\ell_p$ adversarial samples exist because deep learning
models converge to functions similar to linear models.
To improve robustness, they argued models have to be more non-linear.
Based on our framework, linear models are not inherently weak.
When trained, regularized, and parametrized
appropriately they can be robust to some degree, the extent of
which depends on the dataset.
\citet{gilmer2018adversarial} proposed adversarial spheres
as a toy example where a two layer neural network exists with perfect
standard and robust accuracy for non-zero perturbations. Yet,
training a randomly initialized model with gradient descent
and finite data does not converge to a robust model.
Based on our framework, we interpret this as an example
where the implicit bias of gradient descent is not towards
the ground-truth model, even though
there is no misalignment in the architecture.
It would be interesting to understand this implicit bias in future work.

\vspace*{-8pt}
\paragraph{Robustness to $\ell_p$-bounded attacks.}
Robustness is achieved when any perturbation to natural
inputs that changes a classifier's prediction also confuses a human.
$\ell_p$-bounded attacks are the first step in achieving adversarial
robustness.
\citet{tramer2020adaptive} have recently shown many recent robust
models only achieve spurious robustness against $\ell_\infty$ and
$\ell_1$ attacks.
\citet{croce2020reliable} showed that on image classification datasets
there is still a large gap in adversarial robustness to $\ell_p$-bounded
attacks and standard accuracy.
Robustness to multiple $\ell_p$-bounded perturbations through
adversarial training and its trade-offs has also been
analyzed~\citep{tramer2019adversarial,maini2020adversarial}.
\citet{sharif2018suitability, sen2019should} argue that none of
$\ell_0$, $\ell_1$, $\ell_\infty$, or SSIM are a perfect match for human
perception of similarity. That is for any such norm, for any
$\varepsilon$, there exists a perturbation such that humans classify
it differently.
Attacks based on other perceptual similarity metrics
exist~\citep{zhao2020adversarial,liu2019beyond}.
This shows that the quest for adversarial robustness should also
be seen as a quest for understanding human perception.

\vspace*{-8pt}
\paragraph{Robustness through Regularization}
Various regularization
methods have been proposed for adversarial robustness that
penalize the gradient norm and can be studied using the framework
of maximally robust classification.
\citet{lyu2015unified} proposed general $\ell_p$ norm regularization of
gradients.
\citet{hein2017formal} proposed the Cross-Lipschitz penalty
by regularizing the norm
of the difference between two gradient vectors of the function.
\citet{ross2018improving} proposed $\ell_2$ regularization
of the norm of the gradients.
\citet{sokolic2017robust} performed regularization of
Frobenius norm of the per-layer Jacobian.
\citet{moosavi2019robustness} proposed penalizing the curvature
of the loss function.
\citet{qin2019adversarial} proposed encouraging local linearity
by penalizing the error of local linearity.
\citet{simon2019first}
proposed regularization of the gradient norm where
the dual norm of the attack norm is used.
\citet{avery2020adversarial}
proposed Hessian regularization. 
\citet{guo2020connections} showed that some regularization methods
are equivalent or perform similarly in practice.
Strong gradient or curvature regularization methods can
suffer from gradient masking~\citep{avery2020adversarial}.

\vspace*{-8pt}
\paragraph{Certified Robustness.}
Adversarially trained models are empirically harder to attack than
standard models. 
But their robustness is not often provable.
Certifiably robust models seek to close this gap~\citep{hein2017formal,
wong2018provable,
cohen2019certified,%
gowal2018effectiveness,
salman2019convex}.
A model is certifiably
robust if for any input, it also provides an 
$\varepsilon$-certificate that guarantees robustness to any
perturbation within the $\varepsilon$-ball of the input.
In contrast, a maximally robust classifier finds a classifier
that is guaranteed to be robust to maximal $\varepsilon$ while
classifying all training data correctly. That allows
for data dependent robustness guarantees at test time.
In this work, we have not explored standard generalization guarantees.

\vspace*{-10pt}
\section{Conclusion}
\label{sec:conclusion}
\vspace*{-7pt}

We demonstrated that the choice of optimizer,
neural network architecture,
and  regularizer,  significantly affect the adversarial 
robustness of linear neural networks. These results lead us to
a novel Fourier-$\ell_\infty$ attack with controllable spectral
properties applied against deep non-linear CIFAR-10 models.
Our results provide a framework, insights, and directions for improving robustness
of non-linear models through approaches other than adversarial training.

\vspace*{-8pt}
\paragraph{Limitations.} We have not proposed a novel defense for non-linear
models but provided directions and insight.
There are challenges in extending our theory to non-linear models
that needs additional assumptions (See \cref{sec:related,sec:nonlinear}).
There is a growing literature on the implicit bias of
non-linear networks that can be used to extend our
results~\citep{chizat2020implicit,lyu2020gradient,ongie2020function}.
There is a small gap in \cref{fig:conv} between theory and experiment that
might be due to limited training time.

\vspace*{-8pt}
\paragraph{Societal Impact.} We theoretically
connect a security challenge
to the implicit and explicit
biases in machine learning, both with potential negative impacts.
We show that the latter can be used to positively change the former.

\vspace*{-10pt}
\section*{Acknowledgements}
\vspace*{-7pt}
The authors would like to thank Nicholas Carlini, Nicolas Papernot, and
Courtney Paquette for helpful discussions and invaluable feedback. NLR is supported by the Canada CIFAR AI Chair
program.

\bibliographystyle{unsrtnat}
\bibliography{biblio}

\newpage
\appendix

\section{Generalization of the Maximally Robust Classifier}
\label{sec:max_robust_gen}

\begin{defn}[Maximally Robust Classifier with Slack Loss] 
\label{def:eta_max_robust}
Let $\xi \geq 0$ denote a given slack variable.
A maximally robust classifier with slack loss is the solution to
\begin{align}
\label{eq:mrc_slack}
\argmax_{\varphi \in \Phi}
\left\{ \vphantom{i_2^\top} \varepsilon \,|\,
\EE_{(\xx, y)} \max_{\|\ddelta\| \leq \varepsilon}
\zeta(y\varphi(\xx + \ddelta)) \leq \xi\right\}\,.
\end{align}
\end{defn}
This formulation is similar to the saddle-point problem
in that we seek to minimize the expectation of the worst case loss.
The difference is that we also seek to maximize $\varepsilon$.
However, we have introduced another arbitrary variable $\xi$
that is not optimized as part of the problem.
For linear classifiers and the hinge loss,
${\zeta(z)=[1-z]_+}$,
\cref{eq:mrc_slack} can be written as,
\begin{align}
&\argmax_{\ww}\left\{
\vphantom{i_2^\top} \varepsilon \,|\,
\EE_{(\xx, y)} [1-y \ww^\top\xx + \varepsilon \|\ww\|_\ast]_+
\leq \xi\right\}\,,
\end{align}
where $[\cdot]_+$ is the hinge loss,
and the weight penalty term $\|\ww\|_\ast$ is inside the hinge loss.
This subtle difference makes solving the problem more challenging
than weight penalty outside the loss.

Because of the two challenges we noted, we do not study the
maximal robustness with slack loss.

\section{Proofs}
\label{sec:proofs}
\subsection{Proof of \texorpdfstring{\cref{lem:max_robust_to_min_norm_linear}}{Lemma 1}}
\begin{proof}
\label{proof:max_robust_to_min_norm_linear}
We first show that the maximally robust classifier is equivalent to a robust counterpart
by removing $\ddelta$ from the problem,
\begin{align*}
& \argmax_{\ww,b}\,\{ \varepsilon \mid 
 y_i (\ww^\top (\xx_i + \ddelta) + b) > 0
,~\forall i,\, \|\ddelta\| \leq \varepsilon\}\\
&\qquad \text{ (homogeneity of $p$-norm)} \\
&= \argmax_{\ww,b}\,\{ \varepsilon \mid 
 y_i (\ww^\top (\xx_i + \varepsilon\ddelta) + b) > 0
,~\forall i,\, \|\ddelta\| \leq 1\}\\
&\qquad \text{ (if it is true for all $\ddelta$ it is true for the worst of them)}\\ %
& = \argmax_{\ww,b}\,\{ \varepsilon \mid 
 \inf_{\|\ddelta\| \leq 1} y_i (\ww^\top (\xx_i + \varepsilon\ddelta) + b) > 0,~\forall i \}\\
& = \argmax_{\ww,b}\,\{ \varepsilon \mid 
 y_i (\ww^\top \xx_i  + b) + \varepsilon\inf_{\|\ddelta\| \leq 1} \ww^\top \ddelta > 0,~\forall i \}\\
 &\qquad \text{ (definition of dual norm)}\\
& = \argmax_{\ww,b}\,\{ \varepsilon \mid 
 y_i (\ww^\top \xx_i + b) > \varepsilon \|\ww\|_\ast,~\forall i \}\\
\end{align*}

Assuming $\ww\neq 0$, which is a result of linear separability
assumption, we can divide both sides by $\|\ww\|_\ast$
and change variables,
\begin{align*} 
& = \argmax_{\ww,b}\,\{ \varepsilon \mid
 y_i (\ww^\top \xx_i + b) \geq \varepsilon ,~\forall i, \|\ww\|_\ast \leq 1 \}\,,
\end{align*}
where we are also allowed to change $>$ to $\geq$ because any solution
to one problem gives an equivalent solution to the other given
$\ww\neq 0$.

Now we show that the robust counterpart is equivalent to the
minimum  norm classification problem by removing $\varepsilon$.
When the data is linearly separable there exists a solution with $\varepsilon> 0$,
\begin{align*}
& \argmax_{\ww,b}\,\{ \varepsilon \mid 
 y_i (\ww^\top \xx_i + b) > \varepsilon
 \|\ww\|_\ast,~\forall i \}\\
&=\argmax_{\ww,b}\,\left\{ \varepsilon \mid 
 y_i \left(\frac{\ww^\top}{\varepsilon\|\ww\|_\ast} \xx_i + \frac{b}{\varepsilon\|\ww\|_\ast}\right) \geq 1,~\forall i \right\}
\end{align*}

This problem is invariant to any non-zero scaling of $(\ww,b)$,
so with no loss of generality we set $\|\ww\|_\ast=1$.

\begin{align*}
&=\argmax_{\ww,b} \left\{ \vphantom{i_2^\top} \varepsilon \mid
y_i \left(\frac{\ww^\top}{\varepsilon} \xx_i +b\right)\geq 1,\, \forall i,
\|\ww\|_\ast= 1\right\}
\end{align*}

Let $\ww^\prime=\ww/\epsilon$, then the solution to the following problem gives a solution for $\ww$,
\begin{align*}
&\argmax_{\ww^\prime,b} \left\{ \vphantom{i_2^\top} \frac{1}{\|\ww^\prime\|_\ast} \mid y_i (\ww^{\prime\top} \xx_i +b)\geq 1,\, \forall i \right\}\\
&=\argmin_{\ww^\prime,b} \left\{ \vphantom{i_2^\top} \|\ww^\prime\|_\ast \mid y_i (\ww^{\prime\top} \xx_i +b)\geq 1,\, \forall i  \right\}.
\end{align*}

\end{proof}

\subsection{Proof of Maximally Robust to Perturbations Bounded in Fourier Domain (\texorpdfstring{\cref{cor:max_robust_to_min_norm_conv}}{Corollary 2})}
\label{proof:max_robust_to_min_norm_conv}
The proof mostly follows from the equivalence for linear models
in \cref{proof:max_robust_to_min_norm_linear}
by substituting
the dual norm of Fourier-$\ell_1$.
Here, $\AA^\ast$ denotes the complex conjugate transpose,
$\langle \uu,\vv \rangle=\uu^\top\vv^\ast$
is the complex inner product,
$[\bF]_{ik}=\frac{1}{\sqrt{D}}\omega_D^{i k}$
the DFT matrix
where $\omega_D=e^{-j 2\pi/D}$, $j=\sqrt{-1}$.

Let $\|\cdot\|$ be a norm on $\C^{n}$
and $\langle\cdot,\cdot\rangle$ be the complex inner product.
Similar to $\R^{n}$, the associated dual norm is defined as
$\|\ddelta\|_\ast = \sup_{\xx}\{ |\langle \ddelta, \xx \rangle|\; |\; \|\xx\| \leq 1 \}\;$.

\begin{align*}
&\|\mathcal{F}(\ww)\|_1\\
&= \sup_{\|\ddelta\|_\infty\leq 1}
|\langle\mathcal{F}(\ww), \ddelta\rangle|\\
&\qquad \text{ (Expressing DFT as a linear transformation.)}\\
&= \sup_{\|\ddelta\|_\infty\leq 1}
|\langle\bF\ww, \ddelta\rangle|\\
&= \sup_{\|\ddelta\|_\infty\leq 1}
|\langle\ww, \bF^\ast\ddelta\rangle|\\
&\qquad \text{ (Change of variables and $\bF^{-1}=\bF^\ast$.)}\\
&= \sup_{\|\bF\ddelta\|_\infty\leq 1}
|\langle\ww, \ddelta\rangle|\\
&= \sup_{\|\mathcal{F}(\ddelta)\|_\infty\leq 1}
|\langle\ww, \ddelta\rangle|
\,.
\end{align*}

\section{Linear Operations in Discrete Fourier Domain}
\label{sec:fourier_ops}

Finding an adversarial sample with bounded
Fourier-$\ell_\infty$
involves $\ell_\infty$ complex projection to ensure adversarial
samples are bounded, as well as the steepest ascent direction
w.r.t the Fourier-$\ell_\infty$ norm. We also use the complex
projection onto $\ell_\infty$ simplex for proximal gradient method
that minimizes the regularized empirical risk.

\subsection{\texorpdfstring{$\ell_\infty$}{Linf} Complex Projection}

Let $\vv$ denote the $\ell_2$ projection of $\xx\in\C^d$
onto the $\ell_\infty$ unit ball. It can be computed as,
\begin{align}
&\argmin_{\|\vv\|_\infty\leq 1} \frac{1}{2} \|\vv-\xx\|_2^2 \\
&=\{\vv : \forall i,\, \vv_i=
\argmin_{|\vv_i| \leq 1}  \frac{1}{2} |\vv_i-\xx_i|^2\}\,,
\end{align}
that is independent projection per coordinate which can be
solved by 2D projections onto $\ell_2$ the unit ball
in the complex plane.

\subsection{Steepest Ascent Direction w.r.t. Fourier-\texorpdfstring{$\ell_\infty$}{Linf}}

Consider the following optimization problem,
\begin{align}
&\argmax_{\vv : \|\bF\vv\|_\infty\leq 1} f(\vv)\,,
\end{align}
where $\bF\in\C^{d\times d}$ is the
Discrete Fourier Transform
(DFT) matrix and $\bF^\ast=\bF^{-1}$
and $\bF^\ast$ is the conjugate transpose.

Normalized steepest descent direction is defined as (See \citet[Section 9.4]{boyd2004convex}),
\begin{align}
\argmin_\vv\{\nabla \langle f(\ww), \vv\rangle : \|\vv\| = 1 \}\,.
\end{align}

Similarly, we can define the steepest ascent direction,
\begin{align}
&\argmax_{\vv\in\R^d}\{|\langle\nabla f(\ww), \vv\rangle|
: \|\bF\vv\|_\infty = 1\}\\
&\qquad \text{ (Assuming $f$ is linear.)}\\
&\argmax_{\vv\in\R^d}\{|\langle\gg, \bF^\ast\bF\vv\rangle|
: \|\bF\vv\|_\infty = 1\}\\
&\argmax_{\vv\in\R^d}\{|\langle\bF\gg, \bF\vv\rangle|
: \|\bF\vv\|_\infty = 1\}
\end{align}
where $\gg=\nabla f(\ww)$.

Consider the change of variable $\uu=\bF\vv\in\C^{d\times d}$.
Since $\vv$ is a real vector its DFT is  Hermitian,
i.e.\ $\uu_i^\ast=[\uu]_{\overline{-i}}$ for all coordinates $i$ where
$\overline{j}=j \mod d$. Similarly, $\bF\gg$ is Hermitian.
\begin{align}
&\argmax_{\uu\in\C^d : \|\uu\|_\infty = 1}\{|\langle \bF\gg, \uu\rangle|
: \uu_i^\ast=[\uu]_{\overline{-i}}\}\\
&\argmax_{\uu\in\C^d : \forall i, |\uu_i| = 1}\{|[\bF\gg]_i \uu_i|
+ |[\bF\gg]_{\overline{-i}}\, \uu_i^\ast|
: \uu_i^\ast=[\uu]_{\overline{-i}}\}\\
&\argmax_{\uu\in\C^d : \forall i, |\uu_i| = 1}\{|[\bF\gg]_i \uu_i|
+ |[\bF\gg]_i^\ast\uu_i^\ast|
: \uu_i^\ast=[\uu]_{\overline{-i}}\}\\
&\argmax_{\uu\in\C^d : \forall i, |\uu_i| = 1}\{|[\bF\gg]_i \uu_i|
: \uu_i^\ast=[\uu]_{\overline{-i}}\}\\
&\uu_i=[\bF\gg]_i/|[\bF\gg]_i|\,.
\end{align}
and the steepest ascent direction is $\vv_i=\bF^{-1}\uu_i$ which
is a real vector. In practice, there can be non-zero small imaginary
parts as numerical errors which we remove.

\section{Non-linear Maximally Robust Classifiers}
\label{sec:nonlinear}

Recall that the definition of a maximally robust classifier
(\cref{def:max_robust})
handles non-linear families of functions, $\Phi$:
\begin{align*}
    \argmax_{\varphi\in\Phi} \{\varepsilon\,|\,y_i \varphi(\xx_i + \ddelta) > 0 
    ,~\forall i,\,\|\ddelta\| \leq \varepsilon\}\,.
\end{align*}

Here we extend the proof in
\cref{lem:max_robust_to_min_norm_linear}
that made the maximally robust classification tractable
by removing $\ddelta$ and $\varepsilon$ from the problem.
In linking a maximally robust classifier to a minimum norm
classifier when there exists a non-linear transformation, the first
step that requires attention is the following,
\begin{align*}
& \argmax_{\varphi\in\Phi}\,\{ \varepsilon : 
 \inf_{\|\ddelta\| \leq 1} y_i \varphi(\xx_i + \varepsilon\ddelta)
 > 0,~\forall i \} \quad\\
& \neq \argmax_{\varphi\in\Phi}\,\{ \varepsilon : 
 y_i \varphi(\xx_i) + \varepsilon\inf_{\|\ddelta\| \leq 1} \varphi(\ddelta)
 > 0,~\forall i \} \quad
\end{align*}

\begin{lem}[Gradient Norm Weighted Maximum  Margin]
\label{lem:nonlinear}
Let $\Phi$ be a family of 
locally linear classifiers near training data, i.e.\
\begin{align*}
\Phi&=\{\varphi : 
\exists \xi>0,
\forall i,{\|\ddelta\|\leq1},\,
{\varepsilon\in[0,\xi)},\,\\
&{\varphi(\xx_i+\varepsilon\ddelta)}=
{\varphi(\xx_i)}+
{\varepsilon\ddelta^\top\dxy{}{\xx}\varphi(\xx_i)}
\}.
\end{align*}
Then a maximally robust classifier is a solution to the following
problem,
\begin{align*}
& \argmax_{\varphi\in\Phi, \varepsilon \leq \xi}\,\{ \varepsilon : 
y_i \varphi(\xx_i)
> \varepsilon \|\dxy{}{\xx}\varphi(\xx_i)\|_\ast
,~\forall i \}\,.
\end{align*}
\end{lem}

\begin{proof}

\begin{align*}
& \argmax_\varphi\,\{ \varepsilon : 
 \inf_{\|\ddelta\| \leq 1} y_i \varphi(\xx_i + \varepsilon\ddelta)
 \geq 0,~\forall i \} \quad\\
&\qquad \text{ (Taylor approx.)}\\
& = \argmax_\varphi\,\{ \varepsilon : 
\inf_{\|\ddelta\| \leq 1} y_i \varphi(\xx_i)
+ y_i \varepsilon\ddelta \dxy{}{\xx}\varphi(\xx_i)
\geq 0,~\forall i \} \quad\\
& = \argmax_\varphi\,\{ \varepsilon : 
y_i \varphi(\xx_i)
+ \varepsilon\inf_{\|\ddelta\| \leq 1} \ddelta \dxy{}{\xx}\varphi(\xx_i)
\geq 0,~\forall i \} \quad\\
&\qquad\text{(Dual to the local derivative.)}\\
& = \argmax_\varphi\,\{ \varepsilon : 
y_i \varphi(\xx_i)
\geq \varepsilon \|\dxy{}{\xx}\varphi(\xx_i)\|_\ast
,~\forall i \} \quad\\
&\qquad\text{(Assuming constant gradient norm near data.)}\\
& = \argmax_{\varphi:\|\dxy{}{\xx}\varphi(\xx)\|_\ast\leq 1}\,
\{ \varepsilon :  y_i \varphi(\xx_i) \geq \varepsilon  ,~\forall i \}\,.
\end{align*}
\end{proof}

The equivalence in \cref{lem:nonlinear}
fails when $\Phi$ includes functions with non-zero
higher order derivatives within the $\varepsilon$
of the maximally robust classifier.
In practice, this failure manifests itself as various forms of
gradient masking or gradient obfuscation where the model
has almost zero gradient near the data but large
higher-order derivatives~\citep{athalye2018obfuscated}.

Various regularization
methods have been proposed for adversarial robustness that
penalize the gradient norm and can be studied using the framework
of maximally robust classification~\citep{ross2018improving,
simon2019first,
avery2020adversarial,
moosavi2019robustness}
Strong gradient or curvature regularization methods can
suffer from gradient masking~\citep{avery2020adversarial}.

For general family of non-linear functions,
the interplay with implicit bias
of optimization and regularization methods
remains to be characterized.
The solution to the regularized problem in
\cref{def:reg_classifier} is not necessarily unique.
In such cases, the implicit bias of the optimizer biases the
robustness.

\section{Extended Experiments}
\label{sec:exp_ext}
\subsection{Details of Linear Classification Experiments}
\label{sec:linear_ext}
For experiments with linear classifiers,
we sample $n$ training data points from
the $\N(0,\I_d)$, $d$-dimensional standard normal
distribution centered at zero.
We label data points $y=\sign(\ww^\top\xx)$,
using a ground-truth linear separator sampled from  $\N(0,\I_d)$.
For $n<d$, the generated training data is linearly
separable. This setting is similar to a number of recent
theoretical works on the implicit bias of optimization methods in
deep learning and specifically the double descent phenomenon in
generalization~\citep{montanari2019generalization, deng2019doubledescent}.
We focus on robustness against norm-bounded
attacks centered at the training data, in particular,
$\ell_2$, $\ell_\infty$,
$\ell_1$ and Fourier-$\ell_\infty$ bounded attacks.

Because the constraints and the objective
in the minimum norm linear classification problem are convex,
we can use off-the-shelf convex optimization toolbox
to find the solution for small enough $d$ and $n$. We use
the CVXPY library~\citep{diamond2016cvxpy}.
We evaluate the following approaches based on
the implicit bias of optimization:
Gradient Descent (GD), Coordinate Descent (CD),
and Sign Gradient Descent (SignGD)
on fully-connected networks as well as GD on
linear two-layer convolutional networks
(discussed in \cref{sec:implicit_bias}).
We also compare with explicit regularization methods
(discussed in \cref{sec:explicit_reg})
trained using proximal gradient methods~\citep{parikh2013proximal}.
We do not use gradient descent because $\ell_p$ norms
can be non-differentiable at some points
(e.g.\ $\ell_1$ and $\ell_\infty$) and we seek
a global minima of the regularized empirical risk.
We also compare with adversarial training.
As we discussed in \cref{sec:max_robust}
we need to provide the value of maximally robust $\varepsilon$
to adversarial training for finding a maximally robust
classifier. In our experiments,
we give an advantage to adversarial training by providing it with
the maximally robust $\varepsilon$. We also use the steepest
descent direction corresponding to the attack norm
to solve the inner maximization.

For regularization methods a sufficiently
small regularization coefficient achieves
maximal robustness. Adversarial training given
the maximal $\varepsilon$ also converges to the same solution.
We tune all hyperparameters for all methods including learning rate
regularization coefficient and maximum step size in line search.
We provide a list of values in \cref{tab:hparams}.

\begin{table}[t]
    \centering
    \begin{tabular}{m{0.4\linewidth}|m{0.5\linewidth}}
    \toprule
    Hyperparameter
    & Values
    \\
    \midrule
Random seed & 0,1,2 \\
$d$ & 100 \\
$d/n$ & $1, 2, 4, 8, 16, 32$\\
Training steps & $10000$\\
Learning rate & $1\mathrm{e}{-5}$, $3\mathrm{e}{-5}$,
$1\mathrm{e}{-4}$, $3\mathrm{e}{-4}$, $1\mathrm{e}{-3}$, $3\mathrm{e}{-3}$,
$1\mathrm{e}{-2}$, $3\mathrm{e}{-2}$, $1\mathrm{e}{-1}$, $3\mathrm{e}{-1}$,
$1$, $2$, $3$, $6$, $9$, $10$, $20$, $30$, $50$\\
Reg. coefficient &
$1\mathrm{e}{-7}$, $1\mathrm{e}{-6}$, $1\mathrm{e}{-5}$, $1\mathrm{e}{-4}$,
$1\mathrm{e}{-3}$, $1\mathrm{e}{-2}$, $1\mathrm{e}{-1}$, $1$, $10$,
$3\mathrm{e}{-3}$, $5\mathrm{e}{-3}$, $3\mathrm{e}{-2}$, $5\mathrm{e}{-2}$,
$3\mathrm{e}{-1}$, $5\mathrm{e}{-1}$\\
Line search max step & $1$, $10$, $100$, $1000$\\
Adv. Train steps & 10\\
Adv. Train learning rate & 0.1\\
Runtime (line search/prox. method) & $<20$ minutes\\
Runtime (others) & $<2$ minutes\\
    \bottomrule
    \end{tabular}
    \vspace*{5pt}
    \caption{\textbf{Range of Hyperparameters.} Each run uses 2 CPU cores.}
    \label{tab:hparams}
\end{table}

\subsection{Details of CIFAR-10 experiments}

For \cref{fig:tradeoffs_cifar10_eps,fig:tradeoffs_cifar10_trades},
the model is a WRN-28-10.
We use SGD momentum (momentum set to $0.9$)
with a learning rate schedule that warms up from $0$ to LR for $10$ epochs,
then decays slowly using a cosine schedule back to zero over $200$ epochs.
LR is set to 0.1 * BS / 256, where batch size, BS, is set to $1024$.
Experiments runs on Google Cloud TPUv3 over $32$ cores.
All models are trained from scratch and uses the default
initialization from JAX/Haiku.
We use the KL loss and typical adversarial loss for adversarial training.
The inner optimization either maximizes the KL divergence (for TRADES)
or the cross-entropy loss (for AT)
and we use Adam with a step-size of $0.1$.

For the evaluation, we use $40$ PGD steps
(with Adam as the underlying optimizer and step-size $0.1$).
Instead of optimizing the cross-entropy loss, we used the margin-loss~\citep{carlini2017towards}.

For \cref{fig:cifar10_acc}, we evaluate the models in \cref{tab:cifar10_models}
against our Fourier-$\ell_p$ attack with varying $\varepsilon$ in the range
$[0, 8]\times 255$ with step size $0.5$.
We report the largest $\varepsilon$ at which
the robust test accuracy is at most $1\%$ lower than standard test accuracy
of the model. We run the attack for $20$ iterations
with no restarts and use apgd-ce, and apgd-dlr methods from AutoAttack.

\begin{table}[t]
    \centering
    \begin{tabular}{m{0.6\linewidth}|m{0.3\linewidth}}
    \toprule
    Model name & Robust training type
    \\
    \midrule
   \verb|Standard| & - \\
    \midrule
   \verb|Gowal2020Uncovering_70_16_extra| & \texttt{Linf} \\ 
   \verb|Gowal2020Uncovering_28_10_extra| & \texttt{Linf} \\
   \verb|Wu2020Adversarial_extra| & \texttt{Linf} \\
   \verb|Carmon2019Unlabeled| & \texttt{Linf} \\
   \verb|Sehwag2020Hydra| & \texttt{Linf} \\
   \verb|Gowal2020Uncovering_70_16| & \texttt{Linf} \\
   \verb|Gowal2020Uncovering_34_20| & \texttt{Linf} \\
   \verb|Wang2020Improving| & \texttt{Linf} \\
   \verb|Wu2020Adversarial| & \texttt{Linf} \\
   \verb|Hendrycks2019Using| & \texttt{Linf} \\
    \midrule
   \verb|Gowal2020Uncovering_extra| & \texttt{L2} \\ 
   \verb|Gowal2020Uncovering| & \texttt{L2} \\
   \verb|Wu2020Adversarial| & \texttt{L2} \\
   \verb|Augustin2020Adversarial| & \texttt{L2} \\
   \verb|Engstrom2019Robustness| & \texttt{L2} \\
   \verb|Rice2020Overfitting| & \texttt{L2} \\
   \verb|Rice2020Overfitting| & \texttt{L2} \\
   \verb|Rony2019Decoupling| & \texttt{L2} \\
   \verb|Ding2020MMA| & \texttt{L2} \\
    \midrule
   \verb|Hendrycks2020AugMix_ResNeXt|              & \texttt{corruptions} \\ 
   \verb|Hendrycks2020AugMix_WRN| & \texttt{corruptions} \\
   \verb|Kireev2021Effectiveness_RLATAugMixNoJSD| & \texttt{corruptions} \\
   \verb|Kireev2021Effectiveness_AugMixNoJSD| & \texttt{corruptions} \\
   \verb|Kireev2021Effectiveness_Gauss50percent| & \texttt{corruptions} \\
   \verb|Kireev2021Effectiveness_RLAT| & \texttt{corruptions} \\
   \bottomrule
    \end{tabular}
    \vspace*{5pt}
    \caption{List of models evaluated in \cref{fig:cifar10_acc}.}
    \label{tab:cifar10_models}
\end{table}
    
\subsection{Margin Figures}
\label{sec:margin}
A small gap exists between the solution found using CVXPY
compared with coordinate descent.
That is because of limited number of training iterations.
The convergence of coordinate descent
to minimum  $\ell_1$ norm solution is slower than
the convergence of gradient descent to minimum  $\ell_2$ norm
solution. 
There is also a small gap between the solution
of $\ell_1$ regularization and CVXPY. The reason is the
regularization coefficient has to be infinitesimal but
in practice numerical errors prevent us from training
using very small regularization coefficients.

\begin{figure*}[t]
\vspace*{-0.15cm}
\centering
\begin{subfigure}[c]{0.25\textwidth}
\includegraphics[width=\textwidth]{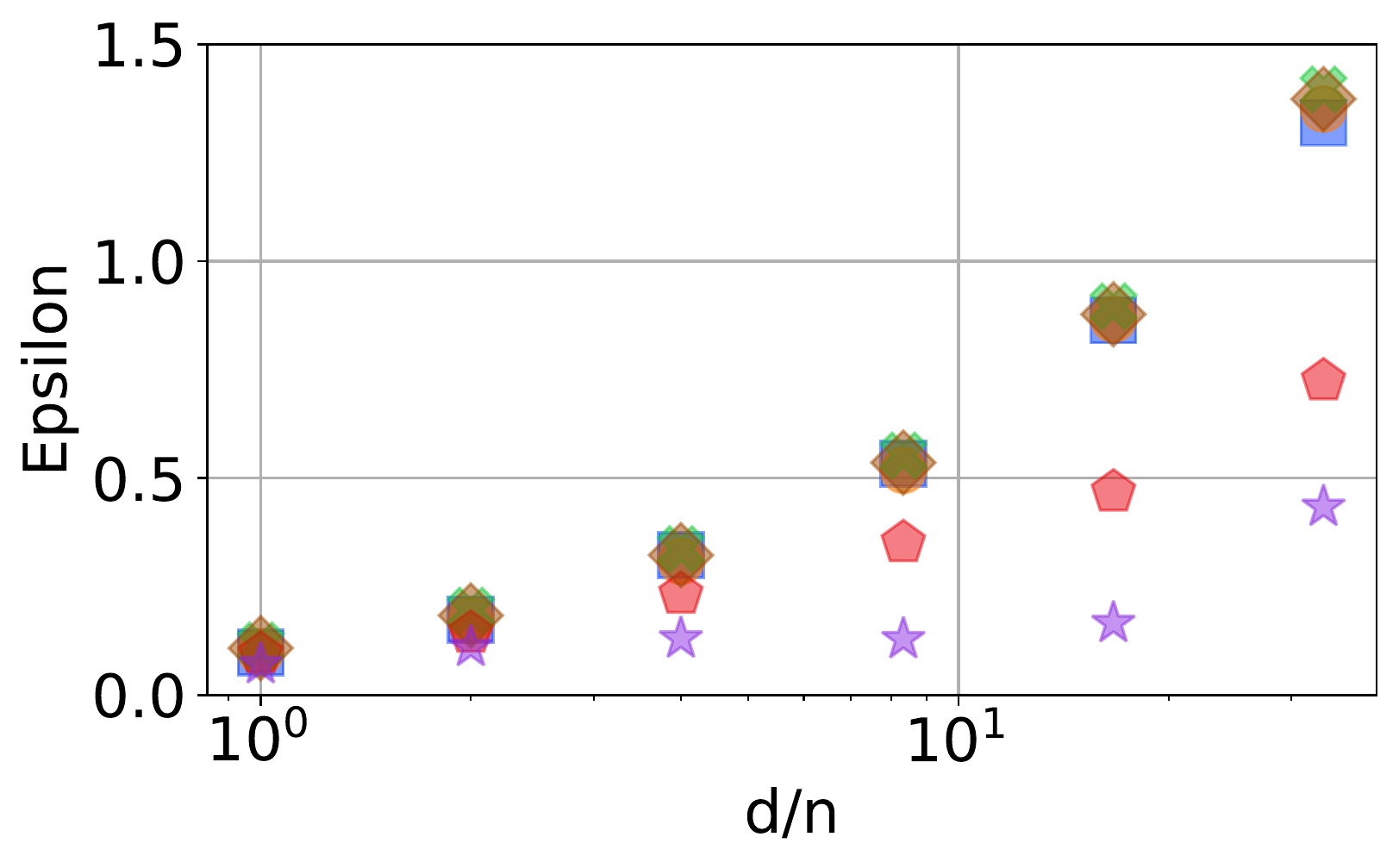}
\caption{$\ell_\infty$ attack}
\end{subfigure}
\hspace{5pt}
\begin{subfigure}[c]{0.25\textwidth}
\includegraphics[width=\textwidth]{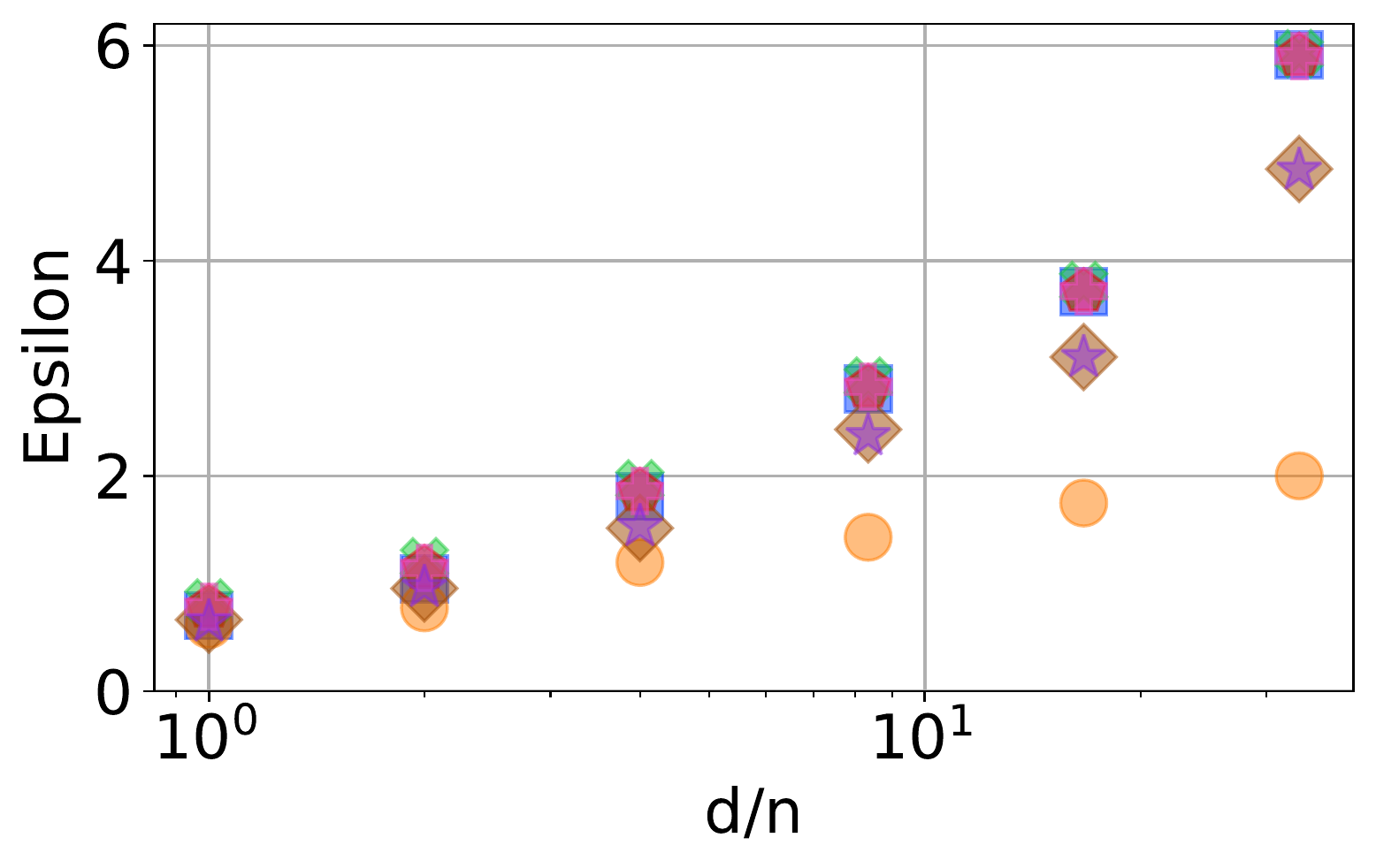}
\caption{$\ell_2$ attack}
\end{subfigure}
\hspace{5pt}
\begin{subfigure}[c]{0.25\textwidth}
\includegraphics[width=\textwidth]{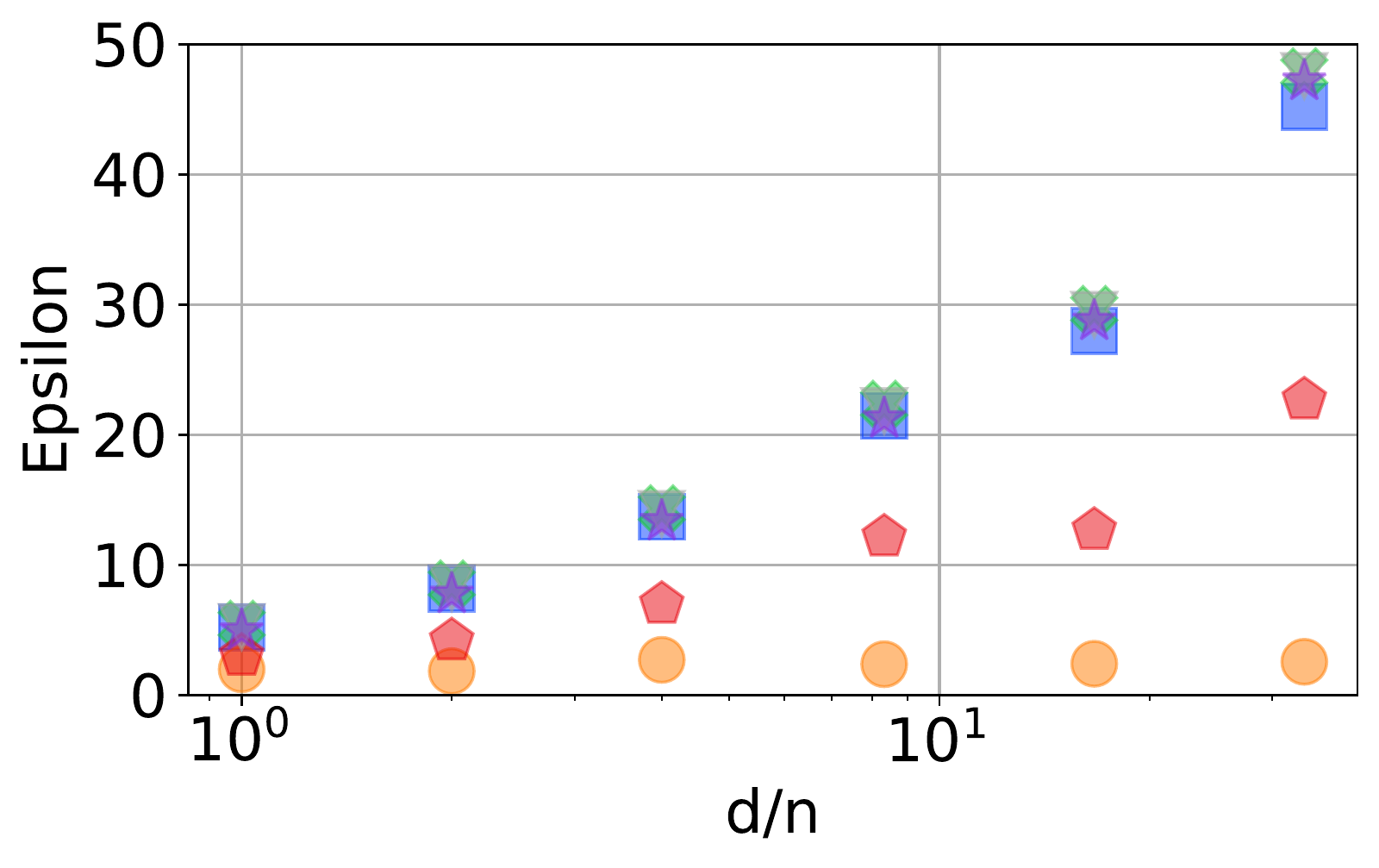}
\caption{$\ell_1$ attack}
\end{subfigure}
\hspace{5pt}
\begin{subfigure}[c]{0.1\textwidth}
\vspace*{-25pt}
\fbox{
\includegraphics[width=\textwidth]{figures/runs_linear_postnorm_randinit_l2/dim_num_train_risk_train_adv_l2_legend.pdf}}
\end{subfigure}

\caption{\textbf{Margin of models in \cref{fig:linear_max_eps}.}
Models are trained to be robust
against $\ell_\infty$, $\ell_2$, $\ell_1$ attacks.
For each attack, there exists one optimizer and one regularization
method that finds the maximally robust classifier. Adversarial
training also finds the solution given the maximal $\varepsilon$.}
\label{fig:linear_margin}
\end{figure*}

\begin{figure}[t]
\begin{subfigure}[b]{\linewidth}
\includegraphics[width=\textwidth]{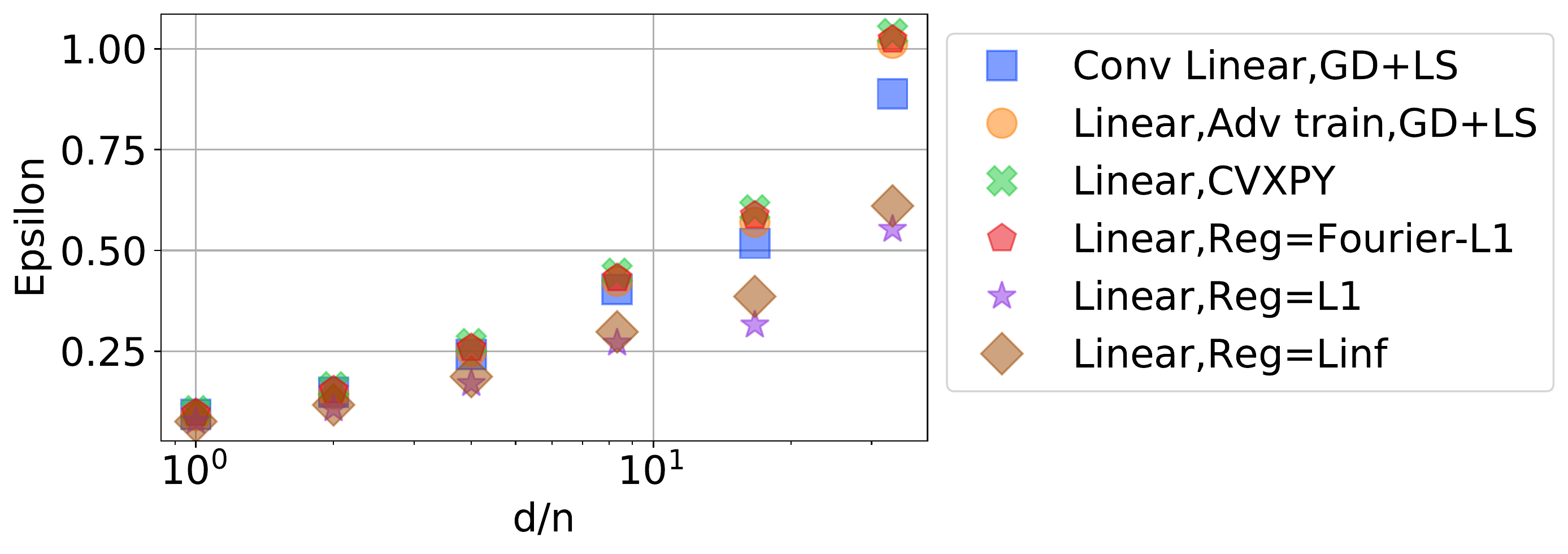}
\end{subfigure}
\caption{\textbf{Fourier-$\ell_1$ margin of
Linear Convolutional Models.}}
\label{fig:conv_margin}
\end{figure}

\subsection{Visualization of Fourier Adversarial Attacks}
\label{sec:fourier_vis}
In
\cref{fig:image_attack_standard,fig:image_attack_carmon,fig:image_attack_augustin}
we visualize adversarial samples for
models available in RobustBench~\citep{croce2020robustbench}.
Fourier-$\ell_\infty$ adversarial samples are qualitatively different
from $\ell_\infty$ adversarial samples as they concentrate on the object.

\begin{figure*}[t]
    \centering
    \begin{subfigure}[b]{\twocolfigwidth}
    \begin{tabular}{c}
    \,\,\,$\xx$\hfill
    $\xx+\ddelta$\hfill
    \,\,\,$\ddelta$ \hfill
    \,\,\,$|\mathcal{F}(\ddelta)|$\hfill\\
    \includegraphics[width=.95\textwidth]{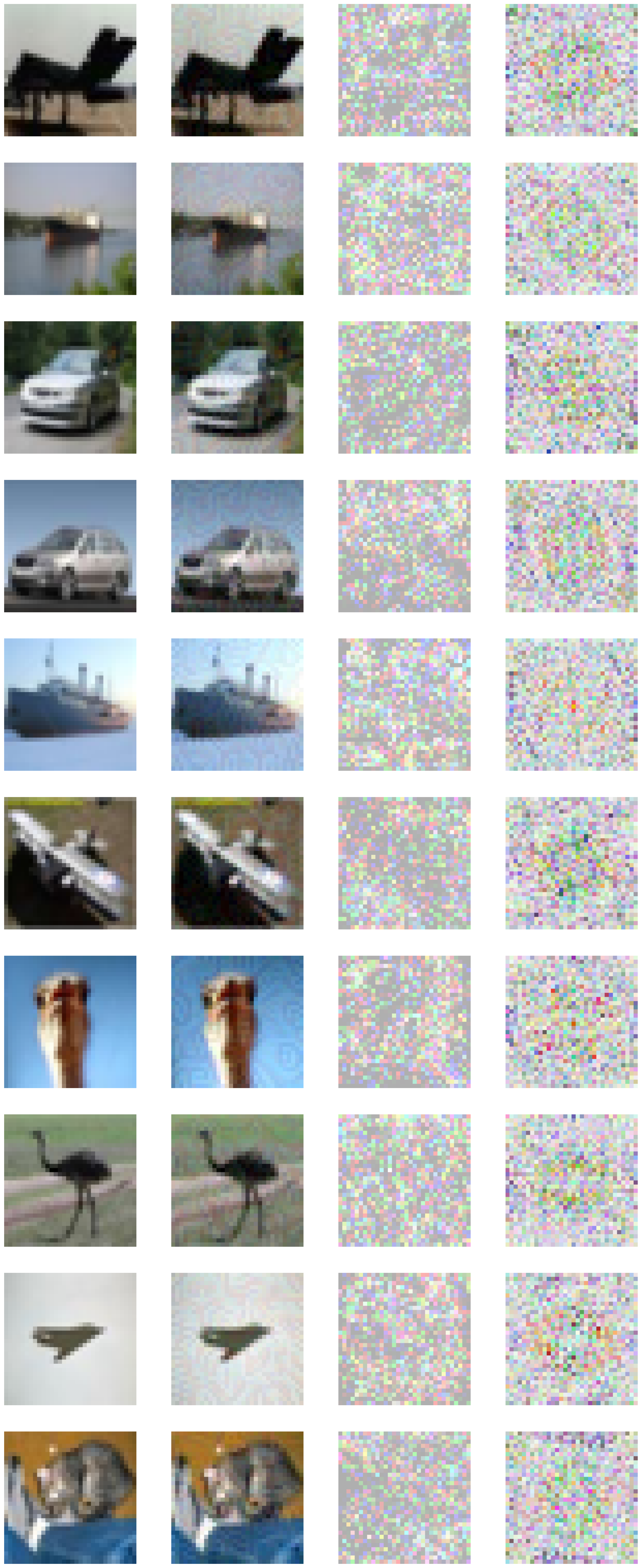}
    \end{tabular}
    \caption{$\ell_\infty$ attack}
    \end{subfigure}
    \hfill
    \begin{subfigure}[b]{\twocolfigwidth}
    \begin{tabular}{c}
    \,\,\,$\xx$\hfill
    $\xx+\ddelta$\hfill
    \,\,\,$\ddelta$ \hfill
    \,\,\,$|\mathcal{F}(\ddelta)|$\hfill\\
    \includegraphics[width=.95\textwidth]{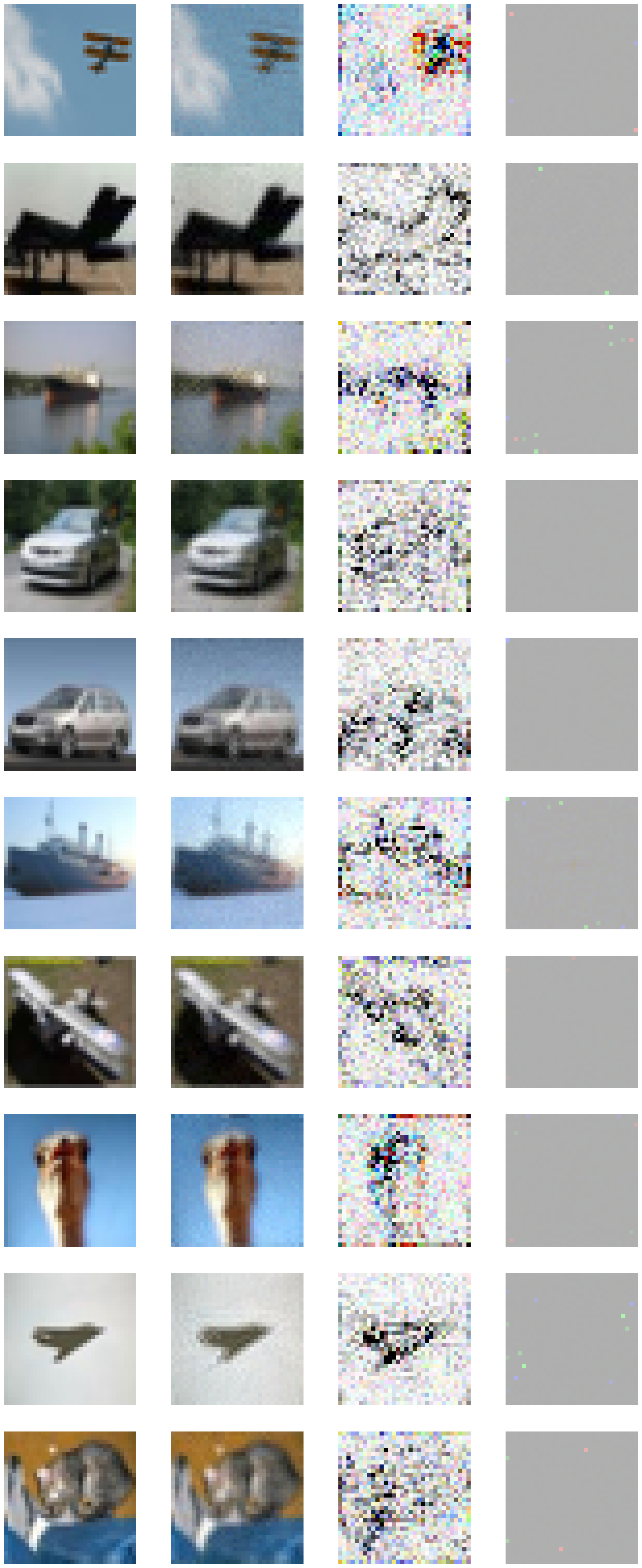}
    \end{tabular}
    \caption{Fourier-$\ell_\infty$ attack}
    \end{subfigure}
    \caption{
    \textbf{Adversarial attacks ($\ell_\infty$ and Fourier-$\ell_\infty$)
    against CIFAR-10 with standard training.}
    WideResNet-28-10 model with standard training.
    The attack methods are APGD-CE and APGD-DLR with default
    hyperparameters in RobustBench. We use $\varepsilon=8/255$ for both attacks.
    Fourier-$\ell_\infty$ perturbations are more concentrated on
    the object.
    Darker color in perturbations means larger magnitude.
    The optimal Fourier attack step is achieved when the
    magnitude in the Fourier domain is equal to the constraints.
    }
    \label{fig:image_attack_standard}
\end{figure*}

\begin{figure*}[t]
    \centering
    \begin{subfigure}[b]{\twocolfigwidth}
    \begin{tabular}{c}
    \,\,\,$\xx$\hfill
    $\xx+\ddelta$\hfill
    \,\,\,$\ddelta$ \hfill
    \,\,\,$|\mathcal{F}(\ddelta)|$\hfill\\
    \includegraphics[width=.95\textwidth]{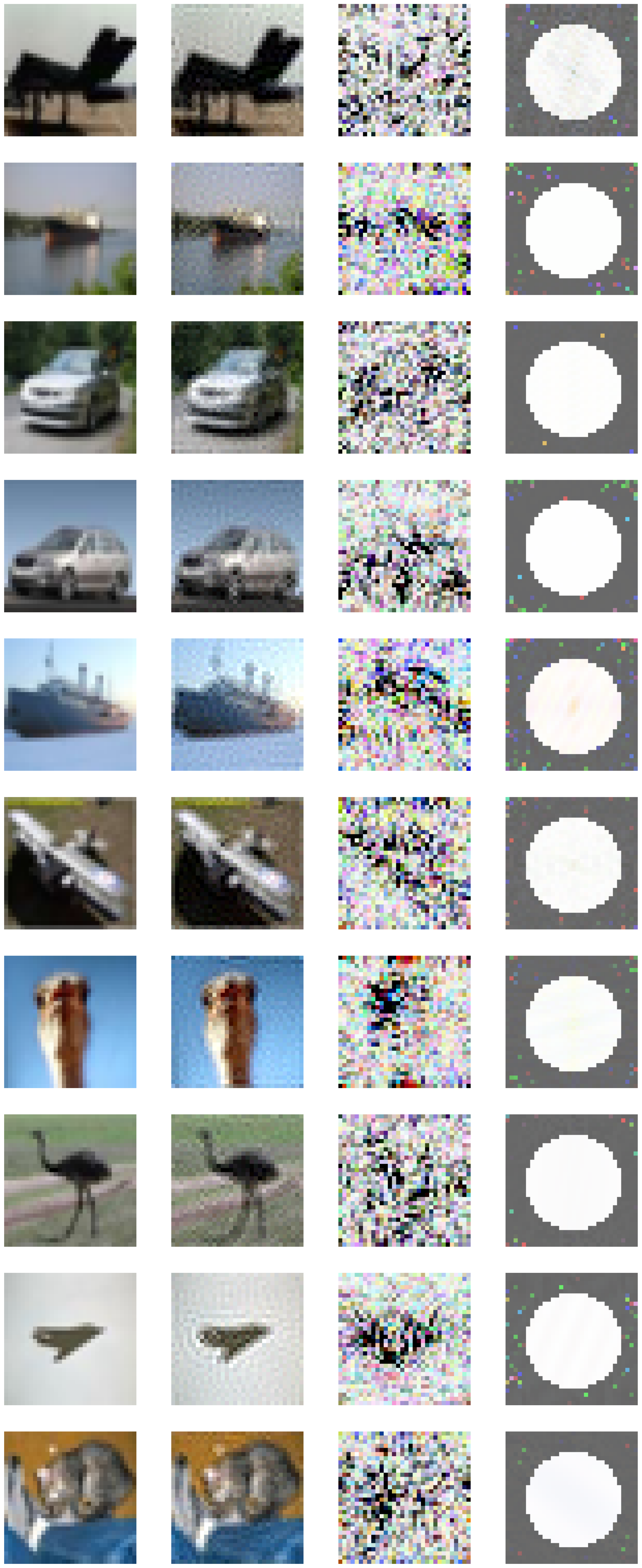}
    \end{tabular}
    \caption{$\ell_\infty$ attack}
    \end{subfigure}
    \hfill
    \begin{subfigure}[b]{\twocolfigwidth}
    \begin{tabular}{c}
    \,\,\,$\xx$\hfill
    $\xx+\ddelta$\hfill
    \,\,\,$\ddelta$ \hfill
    \,\,\,$|\mathcal{F}(\ddelta)|$\hfill\\
    \includegraphics[width=.95\textwidth]{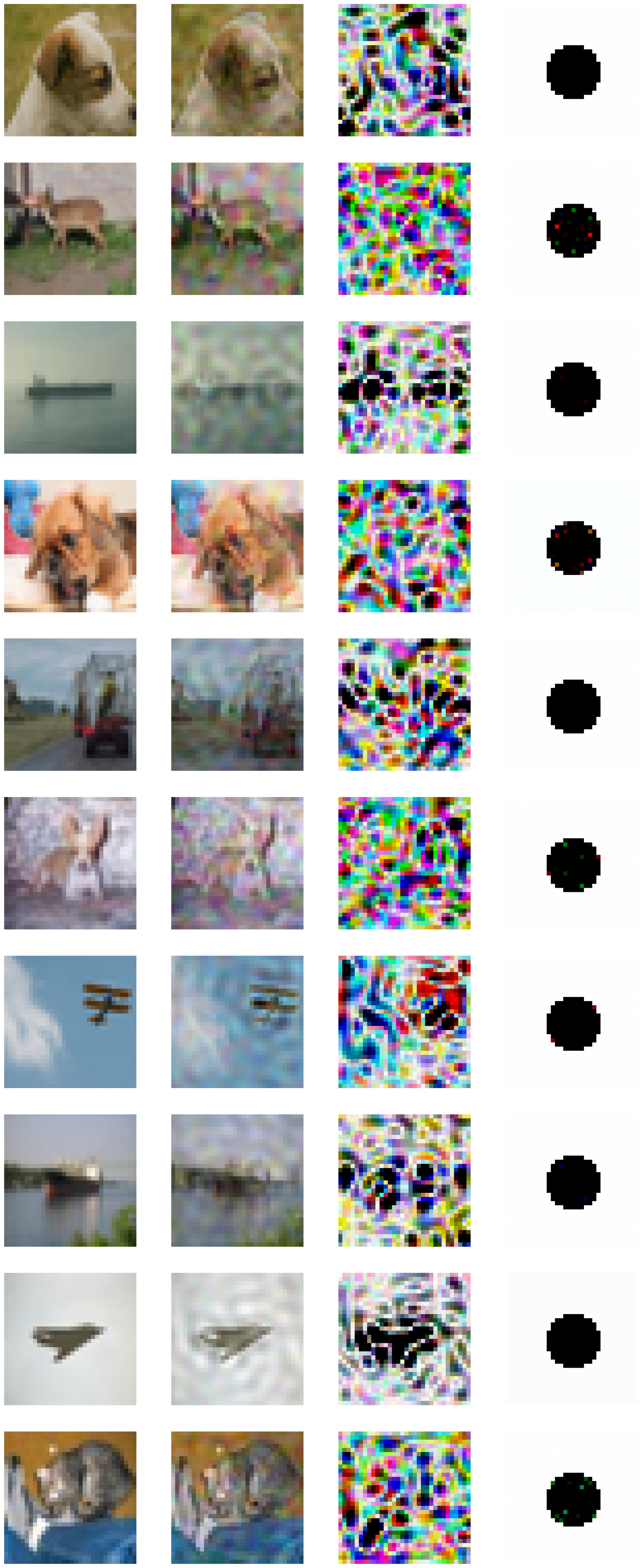}
    \end{tabular}
    \caption{Fourier-$\ell_\infty$ attack}
    \end{subfigure}
    \caption{
    \textbf{Adversarial attacks (High and low frequency Fourier-$\ell_\infty$)
    against CIFAR-10 with standard training.}
    WideResNet-28-10 model with standard training.
    The attack methods are APGD-CE and APGD-DLR with default
    hyperparameters in RobustBench. We use $\varepsilon=15/255,45/255$ respectively for high and low frequency.
    Darker color in perturbations means larger magnitude.
    The optimal Fourier attack step is achieved when the
    magnitude in the Fourier domain is equal to the constraints.
    }
    \label{fig:image_attack_standard_band}
\end{figure*}

\begin{figure*}[t]
    \centering
    \begin{subfigure}[b]{\twocolfigwidth}
    \begin{tabular}{c}
    \,\,\,$\xx$\hfill
    $\xx+\ddelta$\hfill
    \,\,\,$\ddelta$ \hfill
    \,\,\,$|\mathcal{F}(\ddelta)|$\hfill\\
    \includegraphics[width=.95\textwidth]{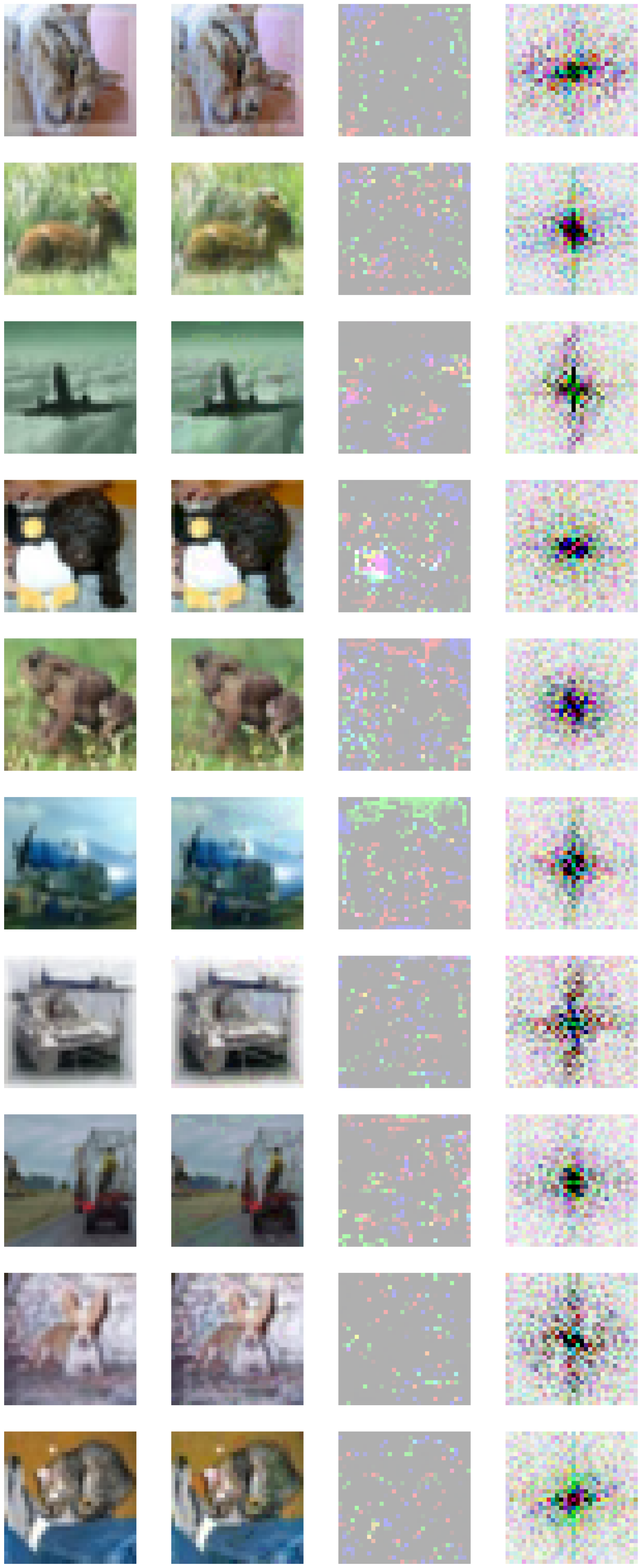}
    \end{tabular}
    \caption{$\ell_\infty$ attack}
    \end{subfigure}
    \hfill
    \begin{subfigure}[b]{\twocolfigwidth}
    \begin{tabular}{c}
    \,\,\,$\xx$\hfill
    $\xx+\ddelta$\hfill
    \,\,\,$\ddelta$ \hfill
    \,\,\,$|\mathcal{F}(\ddelta)|$\hfill\\
    \includegraphics[width=.95\textwidth]{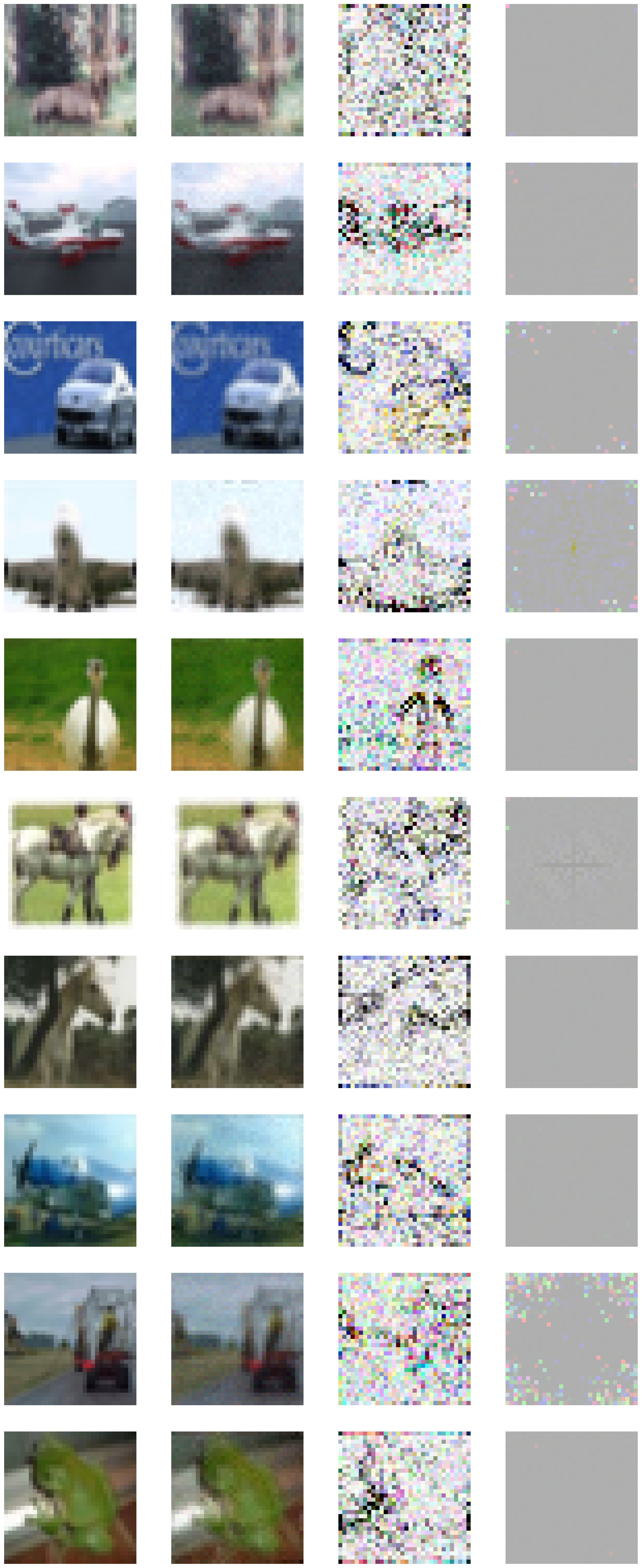}
    \end{tabular}
    \caption{Fourier-$\ell_\infty$ attack}
    \end{subfigure}
    \caption{
    \textbf{Adversarial attacks ($\ell_\infty$ and Fourier-$\ell_\infty$)
    against CIFAR-10 $\ell_\infty$ model of \citep{carmon2019unlabeled}.}
    Adversarially trained model against $\ell_\infty$ attacks.
    The attack methods are APGD-CE and APGD-DLR with default
    hyperparameters in RobustBench. We use $\varepsilon=8/255$ for both attacks.
    Fourier-$\ell_\infty$ perturbations are more concentrated on
    the object.
    Darker color in perturbations means larger magnitude.
    The optimal Fourier attack step is achieved when the
    magnitude in the Fourier domain is equal to the constraints.
    }
    \label{fig:image_attack_carmon}
\end{figure*}

\begin{figure*}[t]
    \centering
    \begin{subfigure}[b]{\twocolfigwidth}
    \begin{tabular}{c}
    \,\,\,$\xx$\hfill
    $\xx+\ddelta$\hfill
    \,\,\,$\ddelta$ \hfill
    \,\,\,$|\mathcal{F}(\ddelta)|$\hfill\\
    \includegraphics[width=.95\textwidth]{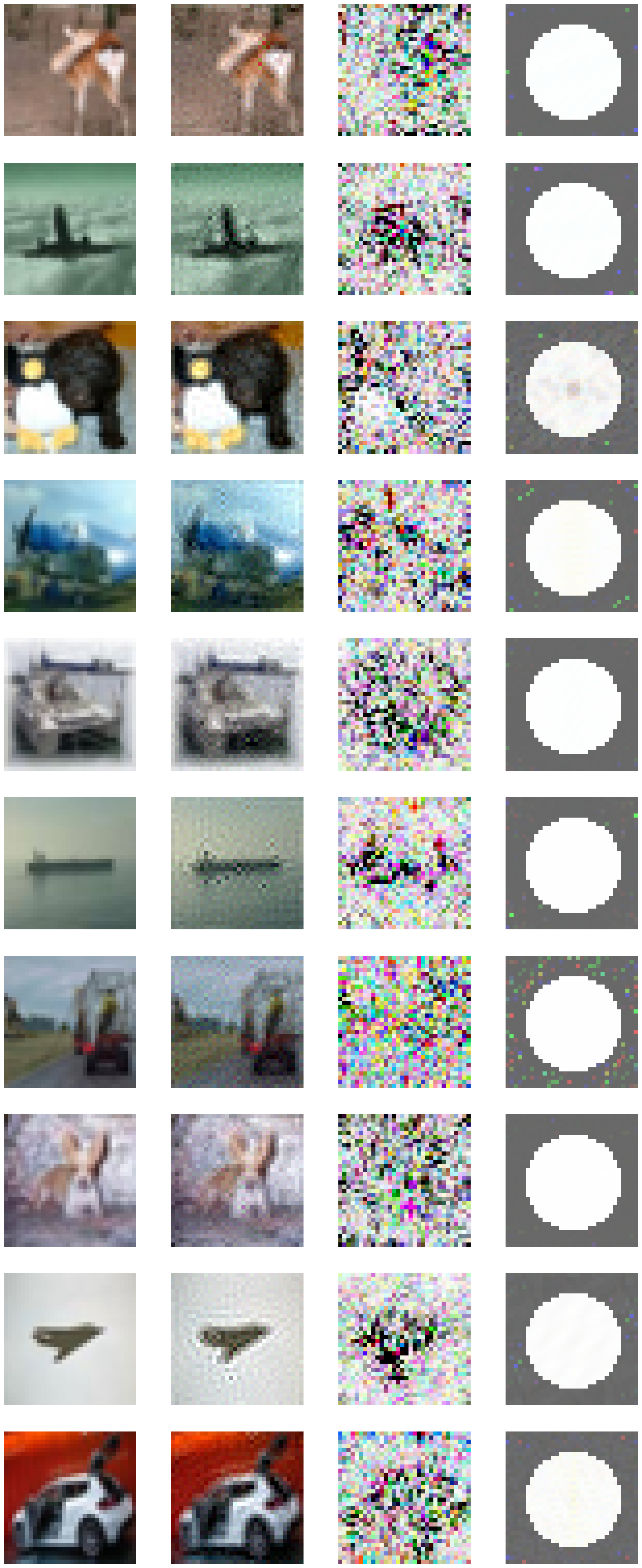}
    \end{tabular}
    \caption{$\ell_\infty$ attack}
    \end{subfigure}
    \hfill
    \begin{subfigure}[b]{\twocolfigwidth}
    \begin{tabular}{c}
    \,\,\,$\xx$\hfill
    $\xx+\ddelta$\hfill
    \,\,\,$\ddelta$ \hfill
    \,\,\,$|\mathcal{F}(\ddelta)|$\hfill\\
    \includegraphics[width=.95\textwidth]{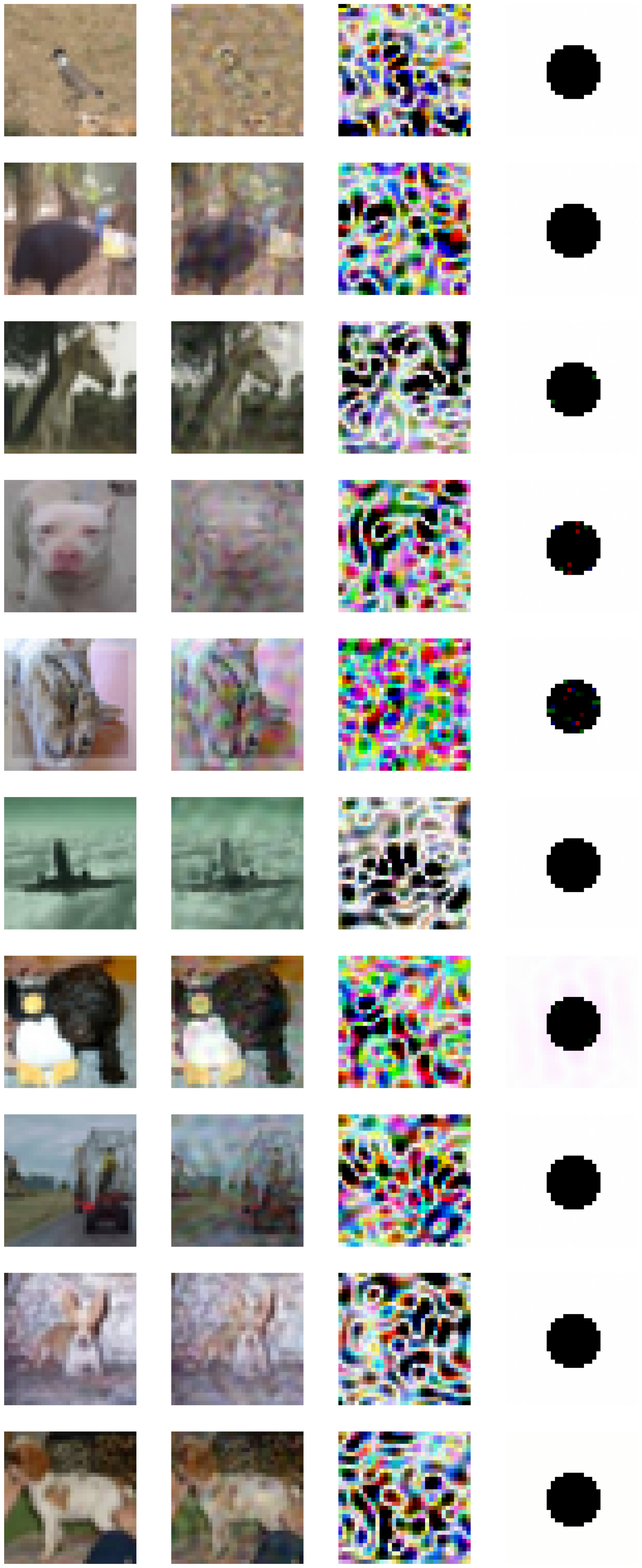}
    \end{tabular}
    \caption{Fourier-$\ell_\infty$ attack}
    \end{subfigure}
    \caption{
    \textbf{Adversarial attacks (High and low frequency Fourier-$\ell_\infty$)
    against CIFAR-10 $\ell_\infty$ model of \citep{carmon2019unlabeled}.}
    WideResNet-28-10 model with standard training.
    The attack methods are APGD-CE and APGD-DLR with default
    hyperparameters in RobustBench. We use $\varepsilon=15/255,45/255$ respectively for high and low frequency.
    Darker color in perturbations means larger magnitude.
    The optimal Fourier attack step is achieved when the
    magnitude in the Fourier domain is equal to the constraints.
    }
    \label{fig:image_attack_carmon_band}
\end{figure*}

\begin{figure*}[t]
    \centering
    \begin{subfigure}[b]{\twocolfigwidth}
    \begin{tabular}{c}
    \,\,\,$\xx$\hfill
    $\xx+\ddelta$\hfill
    \,\,\,$\ddelta$ \hfill
    \,\,\,$|\mathcal{F}(\ddelta)|$\hfill\\
    \includegraphics[width=.95\textwidth]{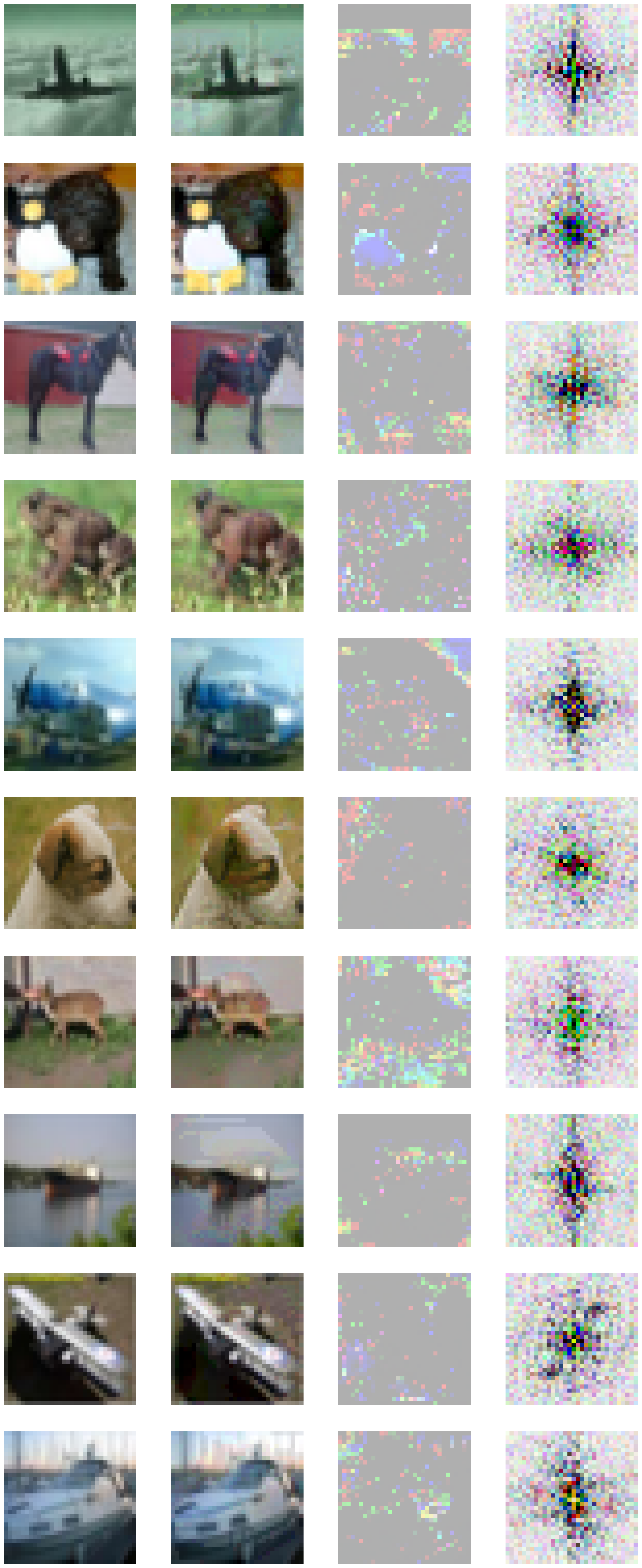}
    \end{tabular}
    \caption{$\ell_\infty$ attack}
    \end{subfigure}
    \hfill
    \begin{subfigure}[b]{\twocolfigwidth}
    \begin{tabular}{c}
    \,\,\,$\xx$\hfill
    $\xx+\ddelta$\hfill
    \,\,\,$\ddelta$ \hfill
    \,\,\,$|\mathcal{F}(\ddelta)|$\hfill\\
    \includegraphics[width=.95\textwidth]{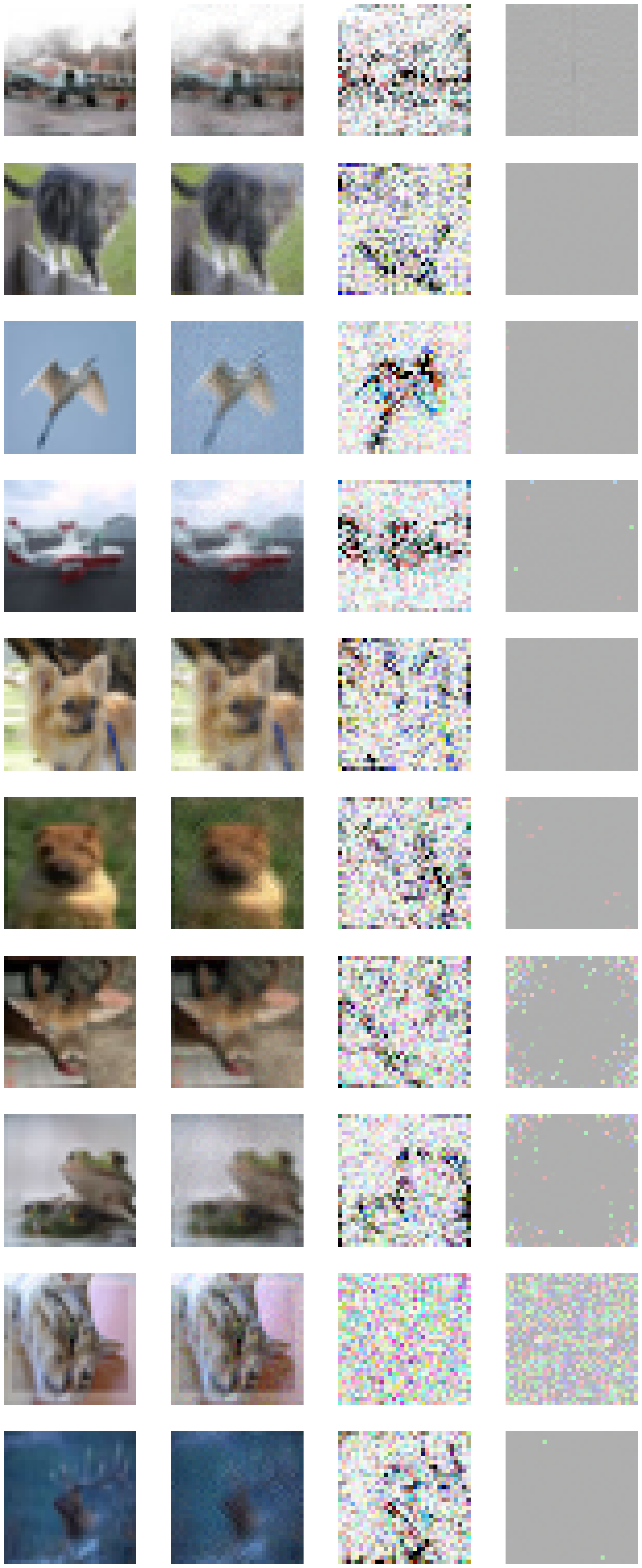}
    \end{tabular}
    \caption{Fourier-$\ell_\infty$ attack}
    \end{subfigure}
    \caption{
    \textbf{Adversarial attacks ($\ell_\infty$ and Fourier-$\ell_\infty$)
    against CIFAR-10 $\ell_2$ model of \citep{augustin2020adversarial}.}
    Adversarially trained model against $\ell_2$ attacks.
    The attack methods are APGD-CE and APGD-DLR with default
    hyperparameters in RobustBench. We use $\varepsilon=8/255$ for both attacks.
    Fourier-$\ell_\infty$ perturbations are more concentrated on
    the object.
    Darker color in perturbations means larger magnitude.
    The optimal Fourier attack step is achieved when the
    magnitude in the Fourier domain is equal to the constraints.
    }
    \label{fig:image_attack_augustin}
\end{figure*}

\begin{figure*}[t]
    \centering
    \begin{subfigure}[b]{\twocolfigwidth}
    \begin{tabular}{c}
    \,\,\,$\xx$\hfill
    $\xx+\ddelta$\hfill
    \,\,\,$\ddelta$ \hfill
    \,\,\,$|\mathcal{F}(\ddelta)|$\hfill\\
    \includegraphics[width=.95\textwidth]{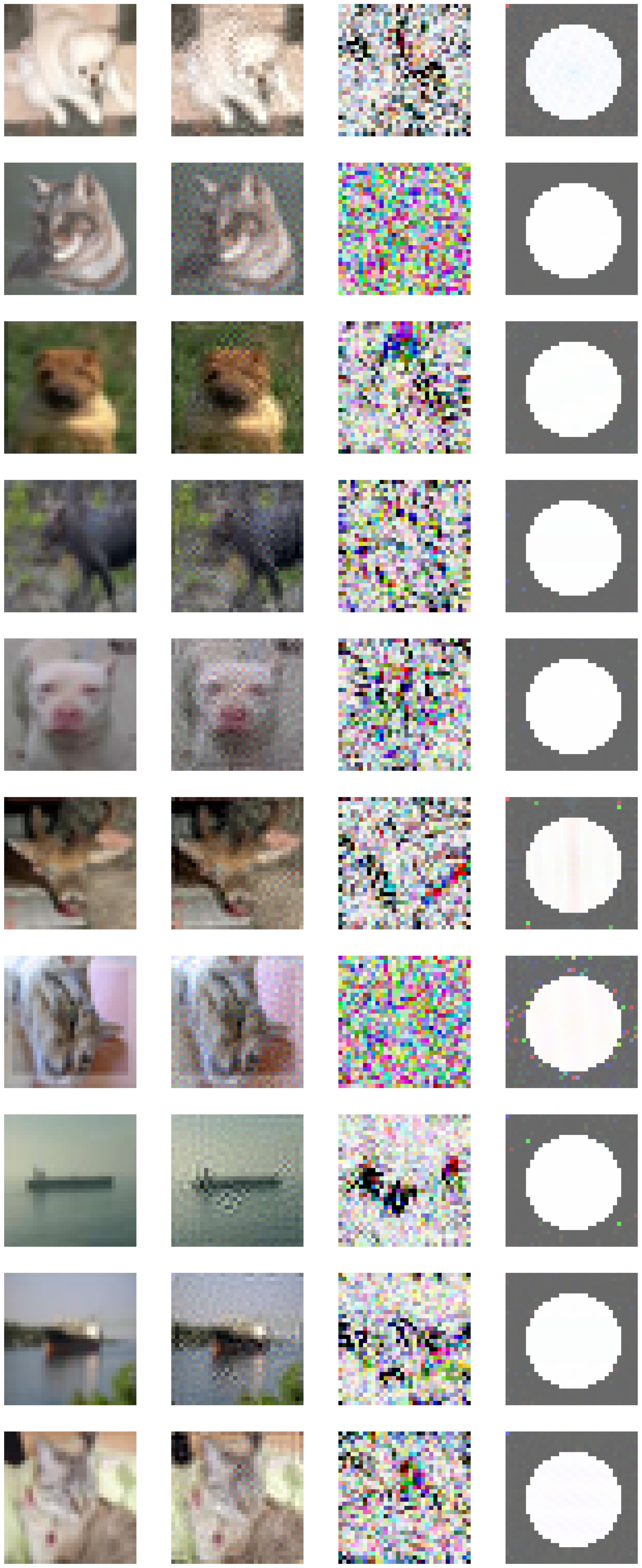}
    \end{tabular}
    \caption{$\ell_\infty$ attack}
    \end{subfigure}
    \hfill
    \begin{subfigure}[b]{\twocolfigwidth}
    \begin{tabular}{c}
    \,\,\,$\xx$\hfill
    $\xx+\ddelta$\hfill
    \,\,\,$\ddelta$ \hfill
    \,\,\,$|\mathcal{F}(\ddelta)|$\hfill\\
    \includegraphics[width=.95\textwidth]{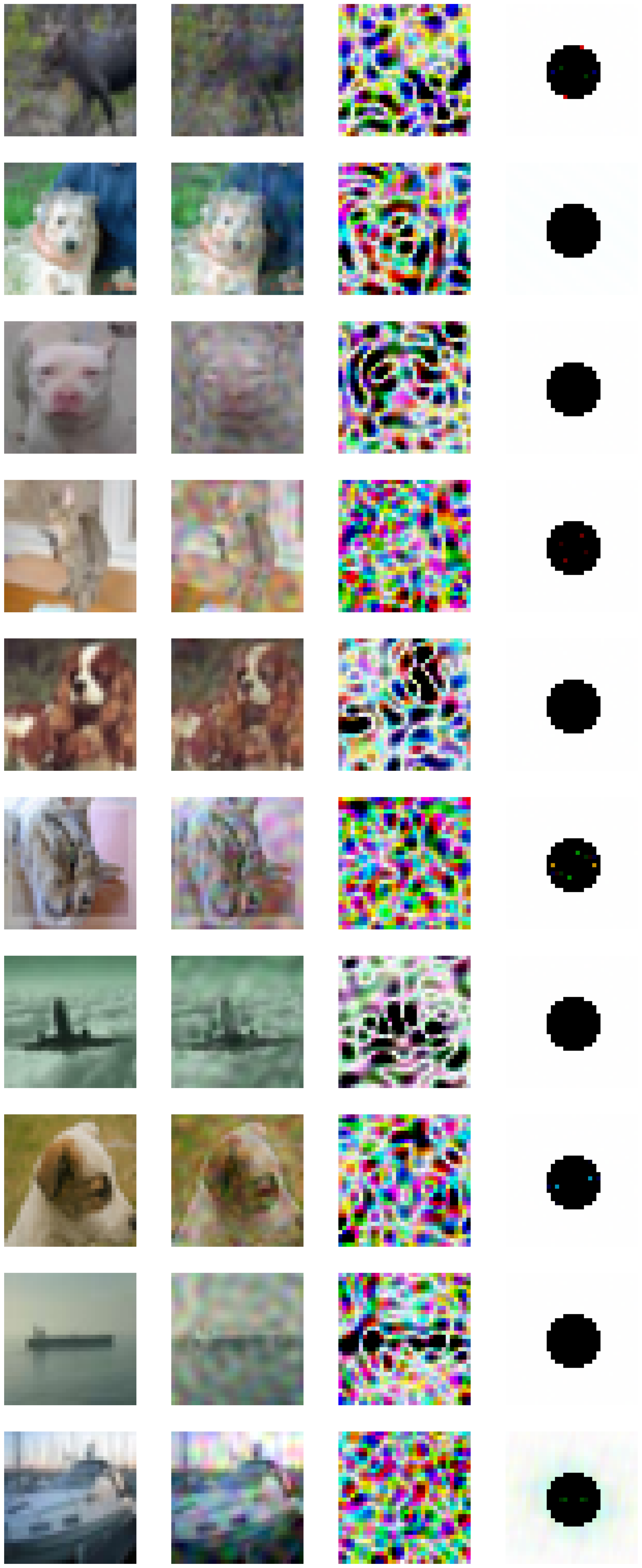}
    \end{tabular}
    \caption{Fourier-$\ell_\infty$ attack}
    \end{subfigure}
    \caption{
    \textbf{Adversarial attacks (High and low frequency Fourier-$\ell_\infty$)
    against CIFAR-10 $\ell_2$ model of \citep{augustin2020adversarial}.}
    WideResNet-28-10 model with standard training.
    The attack methods are APGD-CE and APGD-DLR with default
    hyperparameters in RobustBench. We use $\varepsilon=15/255,45/255$ respectively for high and low frequency.
    Darker color in perturbations means larger magnitude.
    The optimal Fourier attack step is achieved when the
    magnitude in the Fourier domain is equal to the constraints.
    }
    \label{fig:image_attack_augustin_band}
\end{figure*}

\section{Visualization of Norm-balls}
\label{sec:unit_norm_balls}

To reach an intuition of the
norm-ball for Fourier $\ell_\infty$ norm, we visualize
a number of common norm-balls in $3$D in \cref{fig:unit_norm_balls}.
Norm-balls have been visualized in prior
work~\citep{bach2012structured}
but we are not aware of any visualization
of Fourier-$\ell_\infty$. 

\begin{figure}[t]
\centering
\begin{subfigure}[b]{\twocolfigwidth}
\includegraphics[width=\textwidth]{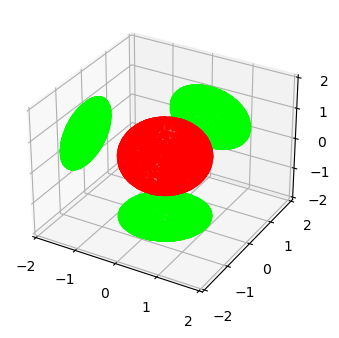}
\caption{$\|\ddelta\|_2=1$}
\end{subfigure}
\begin{subfigure}[b]{\twocolfigwidth}
\includegraphics[width=\textwidth]{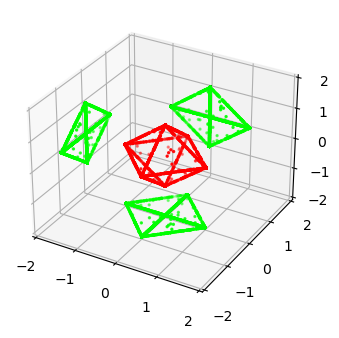}
\caption{$\|\ddelta\|_1=1$}
\end{subfigure}
\begin{subfigure}[b]{\twocolfigwidth}
\includegraphics[width=\textwidth]{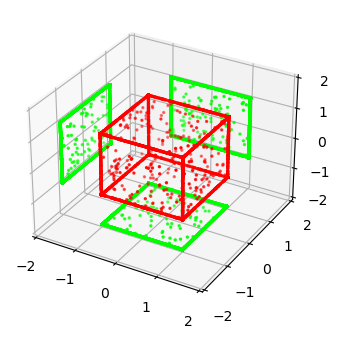}
\caption{$\|\ddelta\|_\infty=1$}
\end{subfigure}
\begin{subfigure}[b]{\twocolfigwidth}
\includegraphics[width=\textwidth]{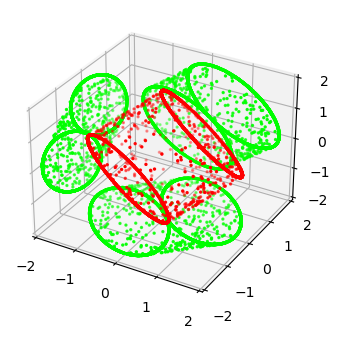}
\caption{$\|\mathcal{F}(\ddelta)\|_\infty=1$}
\end{subfigure}
\caption{Unit norm balls in $3$-D (red) and their $2$-D projections (green). Linear models trained with gradient descent
are maximally robust to $\ell_2$ perturbations.
Two-layer linear convolutional networks trained
with gradient descent are maximally robust to
perturbations with bounded Fourier-$\ell_\infty$.}
\label{fig:unit_norm_balls}
\end{figure}

\end{document}